\documentclass{article}

\PassOptionsToPackage{numbers, compress}{natbib}

\usepackage[preprint]{neurips_2023}

\usepackage[utf8]{inputenc} 
\usepackage[T1]{fontenc}    
\usepackage{hyperref}       
\usepackage{url}            
\usepackage{booktabs}       
\usepackage{amsfonts}       
\usepackage{nicefrac}       
\usepackage{microtype}      
\usepackage{xcolor}         

\usepackage{amsmath}
\usepackage{amssymb}
\usepackage{mathtools}
\usepackage{amsthm}
\usepackage{algorithm}
\usepackage{algorithmic}
\usepackage{enumitem}
\usepackage{wrapfig}
\usepackage{hyperref}
\usepackage{cleveref}
\Crefname{figure}{Fig.}{Figs.}
\Crefname{equation}{Eq.}{Eq.}

\theoremstyle{plain}
\newtheorem{theorem}{Theorem}[section]

\newtheorem{lemma}[theorem]{Lemma}

\theoremstyle{definition}
\newtheorem{definition}[theorem]{Definition}

\theoremstyle{remark}

\newcommand{\yaom}[1]{\textcolor{black}{#1}}
\newcommand{\skx}[1]{\textcolor{black}{#1}}

\title{Probabilistic Modeling: Proving the Lottery Ticket Hypothesis in Spiking Neural Network}

\usepackage{subfigure}  
\usepackage{graphicx}

\author{Man Yao$^{1,2,3}$\thanks{Equal contribution. manyao@stu.xjtu.edu.cn; yhjldl@stu.xjtu.edu.cn}, Yuhong Chou$^{4,2*}$, Guangshe Zhao$^{1}$, Xiawu Zheng$^{3}$, \\ \textbf{Yonghong Tian}$^{5,3}$, \textbf{Bo Xu}$^{2}$, \textbf{Guoqi Li}$^{2}$\thanks{Corresponding author, guoqi.li@ia.ac.cn} \\
~\\
$^{1}$School of Automation Science and Engineering, Xi'an Jiaotong University, Xi'an, Shaanxi, China\\
$^{2}$Institute of Automation, Chinese Academy of Sciences, Beijing, China\\
$^{3}$Peng Cheng Laboratory, Shenzhen, Guangzhou, China\\
$^{4}$College of Artificial Intelligence, Xi'an Jiaotong University, Xi'an, Shaanxi, China\\
$^{5}$Institute for Artificial Intelligence, Peking University, Beijing, China\\
}

\begin{document}

\maketitle

\begin{abstract}
The Lottery Ticket Hypothesis (LTH) states that a randomly-initialized large neural network contains a small sub-network (i.e., winning tickets) which, when trained in isolation, can achieve comparable performance to the large network. LTH opens up a new path for network pruning. Existing proofs of LTH in Artificial Neural Networks (ANNs) are based on continuous activation functions, such as ReLU, which satisfying the Lipschitz condition. However, these theoretical methods are not applicable in Spiking Neural Networks (SNNs) due to the discontinuous of spiking function. We argue that it is possible to extend the scope of LTH by eliminating Lipschitz condition. Specifically, we propose a novel probabilistic modeling approach for spiking neurons with complicated spatio-temporal dynamics. Then we theoretically and experimentally prove that LTH holds in SNNs. According to our theorem, we conclude that pruning directly in accordance with the weight size in existing SNNs is clearly not optimal. We further design a new criterion for pruning based on our theory, which achieves better pruning results than baseline. 
\end{abstract}

\section{Introduction}\label{sec:intro}
By mimicking the spatio-temporal dynamics behaviors of biological neural circuits, Spiking Neural Networks (SNNs) \cite{Maass_1997_LIF,gerstner2014neuronal,Nature_2,yao2022attention} provide a low-power alternative to traditional Artificial Neural Networks (ANNs). The binary spiking communication enables SNNs to be deployed on neuromorphic chips \cite{davies2018loihi,Nature_1,schuman_2022_opportunities} to perform sparse synaptic accumulation for low energy consumption. Given the memory storage limitations of such devices, neural pruning methods are well recognized as one of the crucial methods for implementing SNNs in real-world applications. Pruning redundant weights from an over-parameterized model is a mature and efficient way of obtaining significant compression \cite{han2015learning,hoefler2021sparsity}. 

Recently, Lottery Ticket Hypothesis (LTH), a mile stone is proposed in the literature of network pruning, which asserts that an over-parameterized neural network contains sub-networks that can achieve a similar or even better accuracy than the fully-trained original dense networks by \emph{training only once} \cite{frankle2018the}. A Stronger version of the LTH (SLTH) was then proposed: there is a high probability that a network with random weights contains sub-networks that can approximate any given sufficiently-smaller neural network, \emph{without any training} \cite{ramanujan2020s}. \yaom{SLTH claims to find the target sub-network without training, thus it is a kind of complement to original LTH that requires training \cite{ramanujan2020s}. Meanwhile, SLTH is considered to be ``stronger” than LTH because it claims to require no training \cite{malach2020proving}.} 

\yaom{The effectiveness of LTH in ANNs have been verified by a series of experiments \cite{frankle2018the,zhou2019deconstructing,wang2020pruning,ramanujan2020s}. Furthermore, it has been theoretically proved by several works with various assumptions \cite{malach2020proving,orseau2020logarithmic,pensia2020optimal,da2022proving}, due to its attractive properties in statement.} A line of work is dedicated to designing efficient pruning algorithms based on LTH \cite{you2019drawing,girish2021lottery,chen2020lottery}. But the role of LTH in SNN is almost blank. There is only one work that experimentally verifies that LTH can be applied to SNNs \cite{LTH_SNN_2022}, whether it can also be theoretically established remains unknown. 

\yaom{In this work, we theoretically and experimentally prove that LTH (Strictly speaking, SLTH\footnote{In theoretical proofs, SLTH is considered a stronger version of LTH\cite{malach2020proving, da2022proving}} ) holds in SNNs.} There are two main hurdles. First and foremost, the binary spike signals are fired when the membrane potentials of the spiking neurons exceed the firing threshold, thus the activation function of a spiking neuron is discrete. But all current works \cite{malach2020proving,orseau2020logarithmic,pensia2020optimal,da2022proving} proving LTH in ANNs relies on the Lipschitz condition. Only when the Lipschitz condition is satisfied, the error between the two layers of the neural network will be bounded when they are approximated. Second, brain-inspired SNNs have complex dynamics in the temporal dimension, which is not considered in the existing proofs and increases the difficulty of proving LTH in SNNs.

To bypass the Lipschitz condition, we design a novel probabilistic modeling approach. The approach can theoretically provide the probability that two SNNs behave identically, which is not limited by the complex spatio-temporal dynamics of SNNs. In a nutshell, we establish the following result:

\textbf{Informal version of \Cref{theorem:LTH}} For any given target SNN $\hat{G}$, there is a sufficiently large SNN $G$ with a sub-network (equivalent SNN) $\tilde{G}$ that, with a high probability, can achieve the same output as $\hat{G}$ for the same input,
$$
\mathrm{sup}_{\boldsymbol{S} \in \mathcal{S}}\frac{1}{T} \sum_{t=1}^{T}\left\|\tilde{G}^{t}(\boldsymbol{S}) - \hat{G}^{t}(\boldsymbol{S})\right\|_{2} = 0,
$$
where $t=1,2,\cdots,T$ is the timestep,  $\boldsymbol{S}$ is the input spiking tensor containing only 0 or 1.

\yaom{As a complement to the theoretical proof part, we show experimentally that there are also sub-networks in SNNs that can work without training. Subsequently, it is clear that ANNs and SNNs operate on fundamentally distinct principles about how weight influences the output of the activation layer.} Thus, simply using the weight size as the criterion for pruning \yaom{like ANNs} must not be the optimal strategy. Based on this understanding, we design a new weight pruning criterion for SNNs. We evaluate how likely weights are to affect the firing of spiking neurons, and prune according to the estimated probabilities. In summary, our contributions are as follows:
\begin{itemize}
\item We propose a novel probabilistic modeling method for SNNs that for the first time theoretically establishes the link between spiking firing and pruning (\Cref{sec:modeling}). The probabilistic modeling method is a general theoretical analysis tool for SNNs, and it can also be used to analyze the robustness and other compression methods of SNNs (\Cref{sec:algorithm}).
\item With the probabilistic modeling method, we theoretically prove that LTH also holds in SNNs with binary spiking activations and complex spatio-temporal dynamics (\Cref{sec:prove_LTH}).
\item We experimentally find the good sub-networks without weight training in random initialized SNNs, which is consistent with our theoretical results. (\Cref{sec:exp_verify})
\item We apply LTH for pruning in SNNs and design a new probability-based pruning criterion for it (\Cref{sec:algorithm}). The proposed pruning criterion can achieve better performance in LTH-based pruning methods and can also be exploited in non-LTH traditional pruning methods.
\end{itemize}

\section{Related Work}
\textbf{Spiking neural networks.} The spike-based communication paradigm and event-driven nature of SNNs are key to their energy-efficiency advantages \cite{Nature_2,Deng_2020_rethink_ann_snn,yao_2021_TASNN,yao2022attention}. Spike-based communication makes cheap synaptic Accumulation (AC) the basic operation unit, and the event-driven nature means that only a small portion of the entire network is active at any given time while the rest is idle. In contrast, neurons in ANNs communicate information using continuous values, so Multiply-and-Accumulate (MAC) is the major operation, and generally all MAC must be performed even if all inputs or activations are zeros. SNNs can be deployed on neuromorphic chips for low energy consumption \cite{Nature_1,davies2018loihi,rao2022long}. Thus, spike-based neuromorphic computing has broad application prospects in battery constrained edge computing platforms \cite{davies2021advancing}, e.g., internet of things, smart phones, etc.


\textbf{Pruning in spiking neural networks.} Recent studies on SNN pruning have mostly taken two approaches: 1) Gaining knowledge from ANN pruning's successful experience 2) incorporating the unique biological properties of SNNs. The former technical route is popular and effective. Some typical methods include pruning according to predefined threshold value \cite{neftci2016stochastic,rathi2018stdp,nguyen2021connection}, soft-pruning that training weights and pruning thresholds concurrently \cite{shi2019soft}, etc. The temporal dynamics of SNN are often also taken into consideration in the design of pruning algorithms \cite{guo2020unsupervised,LTH_SNN_2022}. Meanwhile, There have been attempts to develop pruning algorithms based on the similarities between SNNs and neural systems, e.g., regrowth process \cite{kundu2021spike}, spine motility \cite{kappel2015network,bellec2018long}, gradient rewiring \cite{chen2021pruning}, state transition of dendritic spines \cite{chen2022state}, etc. None of these studies, however, took into account the crucial factor that affects network performance: the link between weights and spiking firing. In this work, we use probabilistic modeling to analyze the impact of pruning on spiking firing, and accordingly design a pruning criterion for SNNs.

\begin{figure*}[t]
\begin{center}
\centerline{\includegraphics[width=\linewidth]{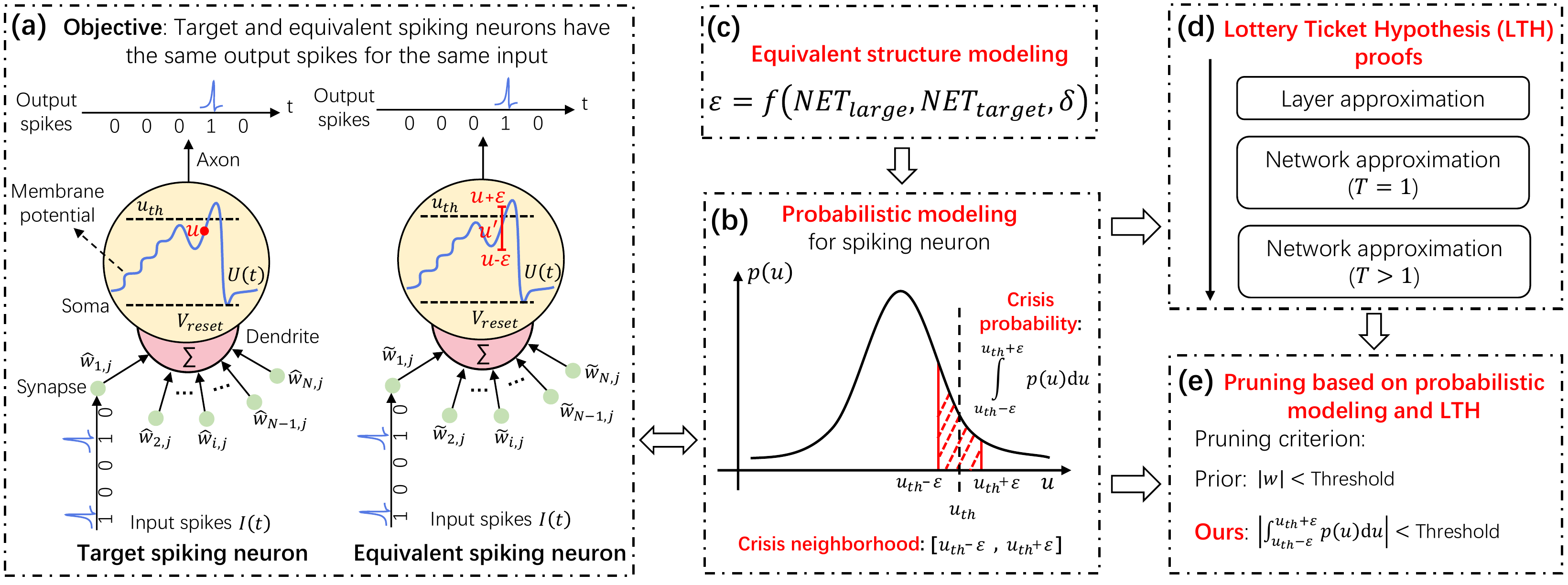}}
\vspace{-2mm}
\caption{\textbf{Overview of this work.} \textbf{(a)}, The goal of spiking neuron approximate. \textbf{(b)}, The firing behavior in the approximation of spiking neurons can be described by the proposed probabilistic modeling approach (\Cref{sec:modeling}). Spikes are fired when the membrane potential $u$ exceeds the firing threshold $u_{th}$. As long as a weight change induces $u$ to fall into the crisis neighborhood, there is a certain probability that the spiking firing will be changed (0 to 1, or 1 to 0). We hope that spiking firing will not change after the redundant weights are pruned. \textbf{(c)}, Equivalent structure modeling. The error between the target and equivalent SNN is related to the network width (\Cref{sec:prove_LTH}). \textbf{(d)}, Proof of LTH in SNN (\Cref{sec:prove_LTH}). \textbf{(e)}, We provide a pruning technique for SNNs (\Cref{sec:algorithm}). New pruning criterion: compute the probability that the firing of a spiking neuron changes when weights are pruned, and pruning according to the rank of the probability.}
\label{fig:abstract}
\end{center}
\vskip -0.3in
\end{figure*}

\section{Probabilistic Modeling for Spiking Neurons}\label{sec:modeling}
No matter how intricate a spiking neuron's internal spatio-temporal dynamics are, ultimately the neuron decide whether to fire a spike or not based on its membrane potential at a specific moment. Thus, the essential question of the proposed probabilistic modeling is how likely is the firing of a spiking neuron to change after changing a weight.

\textbf{Leaky Integrate-and-Fire (LIF)} is one of the most commonly used spiking neuron models \cite{Maass_1997_LIF,gerstner2014neuronal}, because it is a trade-off between the complex spatio-temporal dynamic characteristics of biological neurons and the simplified mathematical form. It can be described by a differential function
\begin{equation}
    \tau\frac{du(t)}{dt}=-u(t)+I(t),  \label{eq:continuous LIF model}
\end{equation}
where $\tau$ is a time constant, and $u(t)$ and $I(t)$ are the membrane potential of the postsynaptic neuron and the input collected from presynaptic neurons, respectively. Solving \Cref{eq:continuous LIF model}, a simple iterative representation of the LIF neuron \cite{Wu_STBP_2018,Neftci_SG_2019} for easy inference and training can be obtained as follow
\begin{align}
\label{eq:mem}
u_{i}^{t, l} &= h_{i}^{t-1, l} + x_{i}^{t, l},\\ \label{eq:fire}
s_{i}^{t, l} &= \operatorname{Hea}(u_{i}^{t, l} - u_{th}),\\\label{eq:temporal_dynamic}
h_{i}^{t, l} &= V_{reset}s_{i}^{t, l} + \beta u_{i}^{t, l}(1-s_{i}^{t, l}),\\
x_{i}^{t, l}&=\sum_{j=1}^{N}w_{ij}^{l}s^{t, l-1}_{j}, 
\label{eq:lif_neuron} 
\end{align}
where $u^{t, l}_{i}$ means the membrane potential of the $i$-th neuron in $l$-th layer at timestep $t$, which is produced by coupling the spatial input feature $x^{t, l}_{i}$ and temporal input $h_{i}^{t-1, l}$, $u_{th}$ is the threshold to determine whether the output spike $s^{t, l}_{i}\in \{0, 1\}$ should be given or stay as zero, $\operatorname{Hea}(\cdot)$ is a Heaviside step function that satisfies $\operatorname{Hea}(x)=1$ when $x\geq0$, otherwise $\operatorname{Hea}(x)=0$, $V_{reset}$ denotes the reset potential which is set after activating the output spiking, and $\beta = e^{-\frac{d t}{\tau}} < 1$ reflects the decay factor. In \Cref{eq:mem}, spatial feature $x_{i}^{t, l}$ can be extracted from the spike $s^{t, l-1}_{j}$ from the spatial output of the previous layer through a linear or convolution operation (i.e., \Cref{eq:lif_neuron}), where the latter can also be regraded as a linear operation \cite{chen2020comprehensive}. $w^{l}_{ij}$ denotes the weight connect from the $j$-th neuron in $(l-1)$-th layer to the $i$-th neuron in $l$-th layer, $N$ indicates the width of the $(l-1)$-th layer. 

The spatio-temporal dynamics of LIF can be described as: the LIF neuron integrates the spatial input feature $x^{t, l}_{i}$ and the temporal input $h_{i}^{t-1, l}$ into membrane potential $u^{t, l}_{i}$, then, the fire and leak mechanism is exploited to generate spatial output $s^{t, l}_{i}$ and the new neuron state $h^{t, l}_{i}$ for the next timestep. Specifically, When $u^{t, l}_{i}$ is greater than the threshold $u_{th}$, a spike is fired and the neuron state $h^{t, l}_{i}$ is reset to $V_{reset}$. Otherwise, no spike is fired and the neuron state is decayed to $\beta u_{i}^{t, l}$. Richer neuronal dynamics \cite{gerstner2014neuronal} can be obtained by adjusting information fusion \cite{yaoglif}, threshold \cite{shaban2021adaptive}, decay \cite{fang2021incorporating}, or reset \cite{diehl2015fast} mechanisms, etc. Note, notations used in this work are summarized in \Cref{app_sec:notation}.

\textbf{Probabilistic modeling of spiking neurons.} Our method is applicable to various spiking neuron models as long as \Cref{eq:fire} is satisfied. In this section, superscript $l$ and $i$ will be omitted when location is not discussed, and superscript $t$ will be omitted when the specific timestep is not important. The crux of probabilistic modeling is how errors introduced by pruning alter the firing of spiking neurons.

We start by discussing only spatial input features. For a spiking neuron, assuming that the temporal input of a certain timestep is fixed, for two different spatial input features $x$ and $x^{\prime} \in [x-\epsilon, x+\epsilon]$, the corresponding membrane potentials are also different, i.e., $u$ and $u^{\prime} \in [u-\epsilon, u+\epsilon]$. Once $u$ and $u^{\prime}$ are located in different sides of the threshold $u_{th}$, the output will be different (see \Cref{fig:abstract}\textbf{a}). This situation can only happen when $u$ is located in \emph{crisis neighborhood} $[u_{th}-\epsilon, u_{th}+\epsilon]$ (see \Cref{fig:abstract}\textbf{b}). That is, suppose membrane potential $u \notin [u_{th}-\epsilon, u_{th}+\epsilon]$, if it changes from $u$ to $u^{\prime}$, and $u^{\prime}\in[u-\epsilon, u+\epsilon]$, the output of the spiking neuron must not change. Consequently, the probability upperbound of a spiking neuron output change (\emph{crisis probability}) is:
\begin{equation}
\label{eq:upper bound}
    \int_{u_{th}-\epsilon}^{u_{th}+\epsilon}p(u)\textrm{d}u,
\end{equation}
where $p(\cdot)$ is the probability density function of the membrane potential distribution. For the case of two independent spiking neurons, if the input is the same, then $\epsilon$ is controlled by the weights of the two neurons (\Cref{fig:abstract}\textbf{a}).

It is reasonable to consider that the membrane potential follows a certain probability distribution. The membrane potential is accumulated by temporal input and spatial input feature, the former can be regarded as a random variable, and the latter is determined by the input spikes and weights. The input spike is a binary random variable according to a firing rate, and usually the weights are also assumed to satisfy a certain distribution. Moreover, some existing works directly assume that the membrane potential satisfies the Gaussian distribution \cite{zheng2021going,guo2022recdis}. In this work, our assumptions about the membrane potential distribution are rather relaxed. Specifically, we give out a class of distributions for membrane potential, and we suppose all membrane potential should satisfy:
\begin{definition}
\label{def:vNFD}
\textbf{m-Neighborhood-Finite Distribution.} For a probability density function $p(\cdot)$ and a value $m$, if \skx{there} exists $\epsilon \skx{>0} $ for the neighborhood $\left[m-\epsilon, m+\epsilon \right]$, in this interval, the max value of function $p(\cdot)$ is finite, we call the distribute function $p(\cdot)$ as m-Neighborhood-Finite Distribution. The symbolic language is as follows
\begin{equation}
\exists \epsilon \skx{>0}, x\in \left[m-\epsilon, m+\epsilon \right], \mathop{sup}\limits_{x}p(x) < +\infty.
\end{equation}
\end{definition}

The upperbound of probability is controllable by controlling the input error $\epsilon$:
\begin{lemma}
\label{lemma:upper bound}
For the probability density function $p(\cdot)$, if it is m-Neighborhood-Finite Distribution, for any $\delta \skx{>0} $, there exists \skx{a constant} $\epsilon \skx{>0} $ \skx{so} that$\int_{m-\epsilon}^{m+\epsilon}p(x)\mathrm{dx} \leq \delta$.
\end{lemma}

Now, we add consideration of the temporal dynamics of spiking neurons. As shown in \Cref{eq:mem} and \Cref{eq:temporal_dynamic}, spiking neurons have a memory function that the membrane potential retains the neuronal state information from all previous timesteps. The spatial input feature $x$ is only related to the current input, and the temporal input $h$ is related to the spatial input features at all previous timesteps. For the convenience of mathematical expression, if there is no special mention to consider the temporal dynamics inside the neuron, we directly write the spiking neuron as $\sigma$. But it should be noted that $\sigma$ actually has complex internal dynamics. In its entirety, a spiking neuron in \Cref{eq:fire} can be denoted as $\sigma^{t}_{x^{1:t-1}}(x^{t})$, where the temporal input $h^{t-1}$ in the membrane potential $u^{t}$ depends on $x^{1:t-1}$, i.e., spatial input features from 1 to $t-1$. 

Using \Cref{def:vNFD}, \Cref{lemma:upper bound}, and the math above, we can determine the probability of differing outputs from the two neurons due to errors in their spatial input features, under varying constraints.Specifically, in \Cref{lemma:singleEFR}, it is assumed that the timestep is fixed and the temporal inputs to the two neurons are the same; in \Cref{lemma:multiTEFR}, we loosen the constraint on the timestep of the two neurons; finally, in \Cref{theorem:theorem_in_chapter_3}, we generalize our results to arbitrary timesteps for two spiking layers ($N$ neurons each layer).

\begin{lemma}
\label{lemma:singleEFR}
At a certain temestep $T$, if the spiking neurons are $u_{th}$-Neighborhood-Finite Distribution and the spatial input features of two neuron got an error upperbound $\epsilon$, and they got the same temporal input $h^{t-1}$, the probability upperbound of different outputs is proportional to $\epsilon$. Formally: 

For two spiking neurons $\hat{\sigma}^{T}$ and $\tilde{\sigma}^{T}$, when $\tilde{h}^{T-1}=\hat{h}^{T-1}$ and $\hat{u}^{T}=\hat{h}^{T-1} + \hat{x}^{T}$ is a random variable follows the $u_{th}$-Neighborhood-Finite Distribution, if $\|\tilde{x}^{T} - \hat{x}^{T}\|\leq \epsilon$, then:
$P\left[\hat{\sigma}^{T}(\hat{x}^{T})\neq\tilde{\sigma}^{T}(\tilde{x}^{T}) \right]\propto \epsilon $
\end{lemma}

\begin{lemma}
\label{lemma:multiTEFR}
Suppose the spiking neurons are $u_{th}$-Neighborhood-Finite Distribution at timestep $T$ and the spatial input features of two corresponding spiking neurons got an error upperbound $\epsilon$ at any timestep, and they got the same temporal input $h^{0}$, if both spiking neurons have the same output at the first $T-1$ timesteps, then the probability upperbound is proportional to $\frac{\epsilon}{1 - \beta}$. Formally: 

For two spiking neurons $\hat{\sigma}^{T}$ and $\tilde{\sigma}^{T}$, when $\tilde{h}^{0}=\hat{h}^{0}$ and $\hat{u}^{T}=\hat{h}^{T-1} + \hat{x}^{T}$ is a random variable follows the $u_{th}$-Neighborhood-Finite Distribution, if $\|\tilde{x}^{t} - \hat{x}^{t}\|\leq \epsilon$ and $\hat{\sigma}^{t}(\hat{x}^{t})=\hat{\sigma}^{t}(\tilde{x}^{t})$ for $t = 1, 2, \cdots, T-1$, then:
$P\left[\hat{\sigma}^{T}(\hat{x}^{T})\neq\tilde{\sigma}^{T}(\tilde{x}^{T}) \right]\propto \frac{\epsilon}{1 - \beta}$. 

\end{lemma}

\begin{theorem}
\label{theorem:theorem_in_chapter_3}
Suppose the spiking layers are $u_{th}$-Neighborhood-Finite Distribution at timestep $t$ and the inputs of two corresponding spiking layers with a width $N$ got an error upperbound $\epsilon$ for each element of spatial input feature vector at any timestep, and they got the same initial temporal input vector $\boldsymbol{h^{0}}$, if there is no different spiking output at the first $T-1$ timesteps, then the probability upperbound is proportional to ${N}\frac{\epsilon}{1 - \beta}$. Formally: 

For two spiking layers $\hat{\sigma}^{T}$ and $\tilde{\sigma}^{T}$, when $\boldsymbol{\tilde{h}^{0}}=\boldsymbol{\hat{h}^{0}}$ and $\boldsymbol{\hat{u}^{t}}=\boldsymbol{\hat{h}^{t-1}} + \boldsymbol{\hat{x}^{t}}$ is a random variable follows the $u_{th}$-Neighborhood-Finite Distribution, if $\|\tilde{x}^{t}_{k} - \hat{x}^{t}_{k}\|\leq \epsilon$ ($k=1, 2, \cdots, N; i=1, 2, \cdots, T$) and $\hat{\sigma}^{t}(\boldsymbol{\hat{x}^{t}})=\hat{\sigma}^{t}(\boldsymbol{\tilde{x}^{t}})$ for $t = 1, 2, \cdots, T-1$, then:
$P\left[\hat{\sigma}^{T}(\boldsymbol{\hat{x}^{T}})\neq\tilde{\sigma}^{T}(\boldsymbol{\tilde{x}^{T}}) \right]\propto {N}\frac{\epsilon}{1 - \beta}$. 
\end{theorem}


\section{Proving Lottery Ticket Hypothesis for Spiking Neural Networks}\label{sec:prove_LTH}
SLTH states that a large randomly initialized neural network has a sub-network that is equivalent to any well-trained network. In this work, the large initialized network, sub-network, well-trained network are named \emph{large network} $G$, \emph{equivalent (pruned) network} $\tilde{G}$, and \emph{target network} $\hat{G}$, respectively. We use the same method to differentiate the weights and other notations in these networks. 

We generally follow the theoretical proof approach in \cite{malach2020proving}. \textbf{Note, the only unique additional assumption we made (considering the characteristics of SNN) is that the membrane potentials follow a class of distribution defined in \Cref{def:vNFD}, which is easy to satisfy.} Specifically, the whole proof in this work is divided into two parts: approximation of linear transformation (\Cref{fig:abstract}\textbf{c}) and approximation of spatio-temporal nonlinear transformation (\Cref{fig:abstract}\textbf{d}), each of which contains three steps. The first part is roughly similar to the proof of LTH in the existing ANN \cite{malach2020proving}. The second part requires our spatio-temporal probabilistic modeling approach in \Cref{sec:modeling}, i.e., \Cref{theorem:theorem_in_chapter_3}.

\textbf{Approximation of linear transformation.} The existing methods for proving LTH in ANNs all first approach a single weight between two neurons by adding a virtual layer (introducing redundant weights), i.e., Step 1. Different virtual layer modeling methods \cite{malach2020proving,orseau2020logarithmic,pensia2020optimal} will affect the width of the network in the final conclusion. All previous proofs introduce two neurons in the virtual layer for approximation. In this work, we exploit only one spiking neuron in the virtual layer for the approximation, given the nature of the binary firing of spiking neurons. We then approximate the weights between one neuron and one layer of neurons (Step 2), and the weights between two layers of neurons (Step 3). Step 2 and Step 3 are the same as previous works.

\emph{Step 1: Approximate single weight by a spiking neuron.} As shown in \Cref{fig:weight_approximate}, a connection weight in SNNs can be approximated via an additional spiking neuron (i.e, a new virtual layer with only one neuron) with two weights, one connect in and one connect out. We set the two weights of the virtual neuron as $v$ and $\tilde{w}$, and the target connection weight is $\hat{w}$. Consequently, the equivalent function for the target weight can be written as $g(s)=\tilde{w}\tilde{\sigma}(v s)$, where the temporal input of the virtual neuron does not need to be considered. Once $v \geq u_{th}$, the virtual neuron will fire no matter how the temporal input is if spatial input $s=1$, while not fire if $s=0$. Thus, the output of the virtual neuron is independent of the temporal input, which is $V_{reset}$ or the decayed membrane potential (neither of which is likely to affect the firing of the virtual neuron). We have: 1) target weight connection: $\hat{g}(s)=\hat{w}s$; 2) equivalent structure: $\tilde{g}(s)=\tilde{w}\tilde{\sigma}(vs)$.
If the weight $v$ satisfy $v \geq u_{th}$, the error between the target weight connection and the equivalent structure is
\begin{equation}
\left\|\tilde{g}(s)-\hat{w}s\right\| = \left\|\tilde{w}s-\hat{w}s\right\|\leq \left\|\tilde{w}-\hat{w}\right\|.
\end{equation}
Once the number of initialized weights in the large network $G$ is large enough, it can choose the weights $\tilde{w}$ and $v$ to make the error approximate to $0$ by probability convergence. Thus, only one spiking neuron is needed for the virtual layer. Formally, we have \Cref{lemma:SWA_appendix}. 

\begin{wrapfigure}[15]{o}{0.5\textwidth} 
\vspace{-6mm}
\begin{center}
\centerline{\includegraphics[width=0.5\columnwidth]{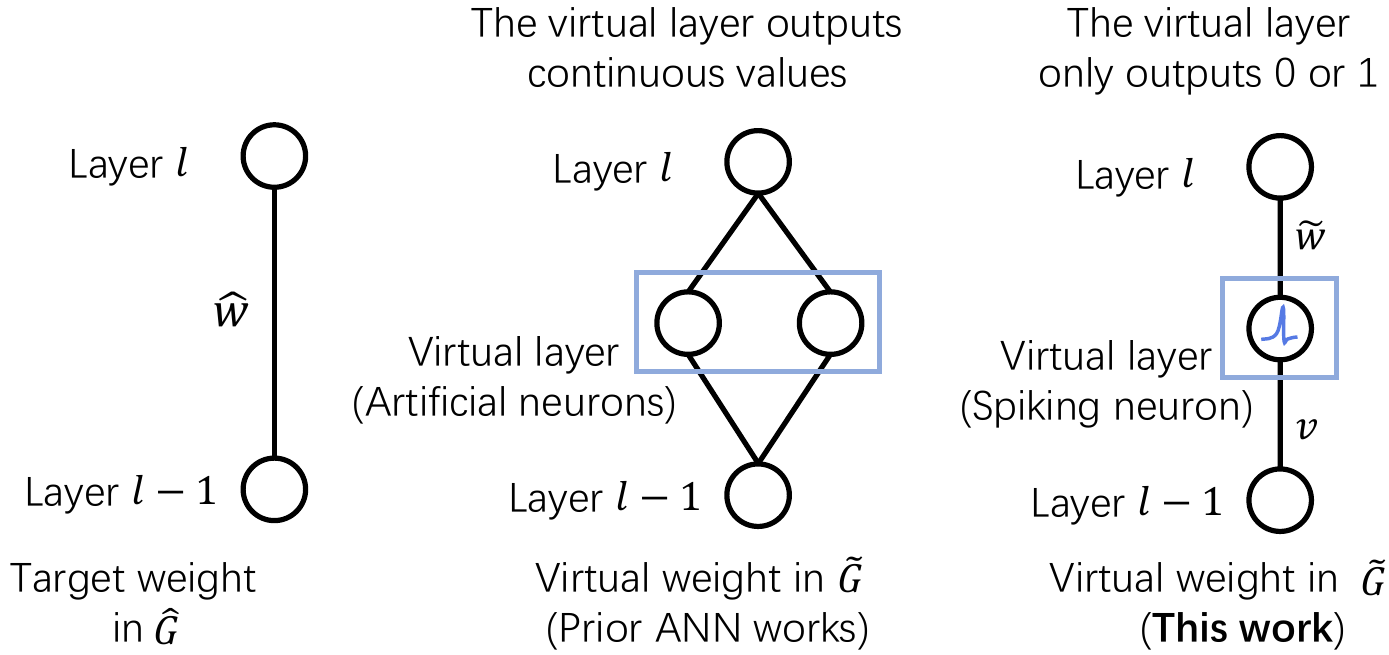}}
\caption{The target weight $\hat{w}$ is simulated in the pruned network $\tilde{G}$ by only one intermediate spiking neuron (this work), due to the spiking neuron only output 0 or 1. In contrast, previous ANN works required two intermediate artificial neurons to simulate the target weights.}
\label{fig:weight_approximate}
\end{center}
\vskip -0.3in
\end{wrapfigure}

\emph{Step 2: Layer weights approximation.} Based on \Cref{lemma:SWA_appendix}, we can extend the conclusion to the case where there is one neuron in layer $l$ and several neurons in layer $l-1$. The lemma is detailed in \Cref{lemma:LWA_appendix}.

\emph{Step 3: Layer to layer approximation.} Finally, we get an approximation between the two layers of spiking neurons. See \Cref{lemma:LLA_appendix} for proof.

\begin{lemma}
\label{lemma:LLA}
\textbf{Layer to Layer Approximation.} Fix the weight matrix $\boldsymbol{\hat{W}} \in [-\frac{1}{\sqrt{N}}, \frac{1}{\sqrt{N}}]^{N \times N}$ which is the connection in target network between a layer of spiking inputs and the next layer of neurons. The equivalent structure is $k$ spiking neurons with $k \times N$ weights $\boldsymbol{V}\in\mathbf{R}^{k\times N}$ connect the input and $\boldsymbol{\tilde{W}}\in\mathbf{R}^{N \times k}$ connect out. All the weights $\tilde{w}_{ij}$ and $v_{ij}$ random initialized with uniform distribution $\mathrm{U}[-1, 1]$ and i.i.d. $\boldsymbol{B}\in\{0, 1\}^{N \times k}$ is the mask for matrix $\boldsymbol{\tilde{W}}$, $\sum_{i, j}\left\|B_{ij}\right\|_{0} \leq N^2, \sum_{j}\left\|B_{ij}\right\|_{0} \leq N$. Then, let the function of equivalent structure be $\tilde{g}(\boldsymbol{s})=(\boldsymbol{\tilde{W}}\odot\boldsymbol{B})\tilde{\sigma}(\boldsymbol{V}\boldsymbol{s})$, where input spiking $\boldsymbol{s}$ is a vector that $\boldsymbol{s}\in\{0, 1\}^{N}$. Then, $\forall \skx{0<}\delta \skx{\leq 1}, \exists \epsilon \skx{>0}$ when $k \geq N\lceil\frac{N}{C_{th}\epsilon}\log{\frac{N}{\delta}}\rceil$ and $C_{th}=\frac{1 - u_{th}}{2}$, there exists a mask matrix $\boldsymbol{B}$ that $\left\|[\tilde{g}(\boldsymbol{s})]_{i}-[\hat{W}\boldsymbol{s}]_{i}\right\| \leq\epsilon$ w.p at least $1-\delta$.
\end{lemma}

\textbf{Approximation of spatio-temporal nonlinear transformation.} Since the activation functions in ANNs satisfy the Lipschitz condition, the input error can control the output error of the activation function. In contrast, the output error is not governed by the input error when the membrane potential of a spiking neuron is at a step position. To break this limitation, in Step 4, we combine the probabilistic modeling method in \Cref{theorem:theorem_in_chapter_3} to analyze the probability of consistent output (only 0 or 1) of spiking layers. Then, in Step 5, we generalize the conclusions in Step 4 to the entire network regardless of the temporal input. Finally, we consider the dynamics of SNNs in the temporal dimension in Step 6. Specifically, we can denote a SNN as follow:
\begin{equation}
G^{t}(\boldsymbol{S}) = G^{t, L}\circ G^{t, L-1}\circ \dots \circ G^{t, 2} \circ G^{t, 1}(\boldsymbol{S}), 
\end{equation}
where $\boldsymbol{S} \in \{0, 1\}^{N \times T}$ is the spatial input of the entire network at all timesteps. $G^{t, l}(\cdot)$ represents the function of network $G$ at $t$-th timestep and $l$-th layer, when considering the specific temporal input, it can be written as:
\begin{equation}
G^{t, l}(\boldsymbol{S}) = \sigma^{t, l}_{\boldsymbol{W}^{l}\boldsymbol{S}^{1:t-1, l}}(\boldsymbol{W}^{l}\boldsymbol{S}^{t, l-1}),
\end{equation}
where spatial input at $t$-th timestep is $\boldsymbol{S}^{t, l-1} \in \{0, 1\}^N$, $\sigma$ denotes the activation function of spiking layer, its subscript $\boldsymbol{S}^{1:t-1, l}$ indicates that the membrane potential at $t$-th timestep is related to the spatial inputs at all previous timesteps. We are not concerned with these details in probabilistic modeling, thus they are omitted in the remaining chapters where they do not cause confusion.

\emph{Step 4: Layer spiking activation approximation.} Compared to the Step 3, we here include spiking activations to evaluate the probability that two different spiking layers (same input, different weights) have the same output. See \Cref{lemma:LAA_appendix} for proof
\begin{lemma}
\label{lemma:LAA}
\textbf{Layer Spiking Activation Approximation.} Fix the weight matrix $\boldsymbol{\hat{W}} \in [-\frac{1}{\sqrt{N}}, \frac{1}{\sqrt{N}}]^{N \times N}$ which is the connection in target network between a layer of spiking inputs and the next layer of neurons. The equivalent structure is $k$ spiking neurons with $k \times N$ weights $\boldsymbol{V}\in\mathbf{R}^{k\times N}$ connect the input and $\boldsymbol{\tilde{W}}\in\mathbf{R}^{N \times k}$ connect out. All the weights $\tilde{w}_{ij}$ and $v_{ij}$ random initialized with uniform distribution $\mathrm{U}[-1, 1]$ and i.i.d. $\boldsymbol{B}\in\{0, 1\}^{N \times k}$ is the mask for matrix $\boldsymbol{V}$, $\sum_{i, j}\left\|B_{ij}\right\|_{0} \leq N^2, \sum_{j}\left\|B_{ij}\right\|_{0} \leq N$. Then, let the function of equivalent structure be $\tilde{g}(\boldsymbol{s})=\tilde{\sigma}((\boldsymbol{\tilde{W}}\odot \boldsymbol{B})\tilde{\sigma}(\boldsymbol{V}\boldsymbol{s}))$, where input spiking $\boldsymbol{s}$ is a vector that $\boldsymbol{s}\in\{0, 1\}^{N}$. $C$ is the constant depending on the supremum probability density of the dataset of the network. Then, $\forall \delta, \exists \epsilon$ when $k \geq N^2\lceil\frac{N}{C_{th}\epsilon}\log{\frac{N^2}{\delta-NC\epsilon}}\rceil$, there exists a mask matrix $\boldsymbol{B}$ that $\left\|\tilde{g}(\boldsymbol{s})-\hat{\sigma}(\boldsymbol{\hat{W}}\boldsymbol{s})\right\| = 0$ w.p at least $1-\delta$.
\end{lemma}

\emph{Step 5: Network approximation ($T=1$).} The lemma is generalized to the whole SNN in \Cref{lemma:ALA_appendix}.

\emph{Step 6: Network approximation ($T > 1$).} If the output of two SNNs in the first $T-1$ timestep is consistent, and we assume that the error of spatial input features at all timesteps has an upperbound, there is a certain probability that the output of the two SNNs in timestep $T$ is also the same. See detail proof in \Cref{theorem:LTH_appendix} and detail theorem as follow \Cref{theorem:LTH}:
\begin{theorem}
\label{theorem:LTH}
\textbf{All Steps Approximation.} Fix the weight matrix $\boldsymbol{\hat{W}}^{l} \in [-\frac{1}{\sqrt{N}}, \frac{1}{\sqrt{N}}]^{N \times N}$ which is the connection in target network between a layer of spiking inputs and the next layer of neurons. The equivalent structure is $k$ spiking neurons with $k \times N$ weights $\boldsymbol{V}^{l}\in\mathbf{R}^{k\times N}$ connect the input and $\boldsymbol{\tilde{W}}^{l}\in\mathbf{R}^{N \times k}$ connect out. All the weights $\boldsymbol{\tilde{w}}^{l}_{ij}$ and $v^{l}_{ij}$ random initialized with uniform distribution $\mathrm{U}[-1, 1]$ and i.i.d. $\boldsymbol{B}^{l}\in\{0, 1\}^{k \times N}$ is the mask for matrix $\boldsymbol{V}^{l}$, $\sum_{i, j}\left\|B^{l}_{ij}\right\|_{0} \leq N^2, \sum_{j}\left\|B^{l}_{ij}\right\|_{0} \leq N$. Then, let the function of equivalent network at timestep $t$ be $\tilde{G}^{t}(\boldsymbol{S})=\tilde{G}^{t, L}\circ\tilde{G}^{t, L-1}\circ\cdots\circ\tilde{G}^{t, 1}(\boldsymbol{S}) $ and $\tilde{G}^{t, l} = \tilde{\sigma}^{t}((\boldsymbol{\tilde{W}}^{l}\odot\boldsymbol{B}^{l})\tilde{\sigma}^{t}(\boldsymbol{V}^{l}\boldsymbol{S}^{t}))$, where input spiking $\boldsymbol{S}$ is a tensor that $\boldsymbol{S}\in\{0, 1\}^{N\times T}$. And the target network at timestep $t$ is $\hat{G}^{t}(\boldsymbol{S})=\hat{G}^{t, L}\circ\hat{G}^{t, L-1}\circ\cdots\circ\hat{G^{t, 1}}(\boldsymbol{S})$, where $\hat{G}^{t, l}(\boldsymbol{S})=\hat{\sigma}^{t}(\hat{W}^{l}\boldsymbol{S}^{t})$. $l=1, 2, \cdots, L, t=1, 2, \cdots, T$. $C$ is the constant depending on the supremum probability density of the dataset of the network. Then, $\forall \skx{0<}\delta \skx{\leq 1}, \exists \epsilon \skx{>0}$ when$ 
k \geq N^2\lceil\frac{N}{C_{th}\epsilon}\log{\frac{N^2 L}{\delta-NCLT\epsilon}}\rceil, 
$ there exists a mask matrix $\boldsymbol{B}$ that  $\left\|\tilde{G}(\boldsymbol{S})-\hat{G}(\boldsymbol{S})\right\| = 0,$ w.p at least $1-\delta$.
\end{theorem}

\section{Searching Winning Tickets from SNNs without Weight Training}\label{sec:exp_verify}

By applying the top-$k\%$ sub-network searching (\texttt{edge-popup}) algorithm\cite{ramanujan2020s}, we empirically verified SLTH as the complement of the theoretical proof in \Cref{sec:prove_LTH}. Spiking neurons do not satisfy the Lipschitz condition, but the ANN technique can still be employed since the surrogate gradient of the SNN during BP training \cite{Wu_STBP_2018}. We first simply introduce the search algorithm, then show the results. 

The sub-network search algorithm first proposed by \cite{ramanujan2020s} provides scores for every weight as a criterion for sub-network topology choosing. In each step of forward propagation, the top-$k\%$ score weights are chosen while others are masked by $0$. It means the network used for classification is actually a sub-network of the original random initialized network. During the BP training process, the scores are randomly initialized as weights at the beginning, and then the gradient of the score is updated while the gradient of weights is not recorded (weights are not changed). The updating rule of scores can be formally expressed in mathematical notations as follow:
\begin{equation}
    s_{uv} \leftarrow s_{uv} + \alpha \frac{\partial\mathcal{L}}{\partial\mathcal{I}_{v}}S_{u}w_{uv},
\end{equation}
Where $s_{uv}$ and $w_{uv}$ denote the score and corresponding weight that connect from spiking neuron $u$ to $v$, respectively, $S_{u}\in\{0, 1\}$ indicates the output of neuron $u$, $\mathcal{L}$ and $\mathcal{I}_{v}$ are loss function value of the network and the input value of the neuron $v$, and $\alpha$ is the learning rate.

\begin{wrapfigure}[22]{r}{0.4\textwidth}
    \vspace{-6mm}
	\centering  
	\subfigbottomskip=-1pt 
	\subfigcapskip=-5pt 
	\subfigure[CIFAR10]{
		\includegraphics[width=\linewidth]{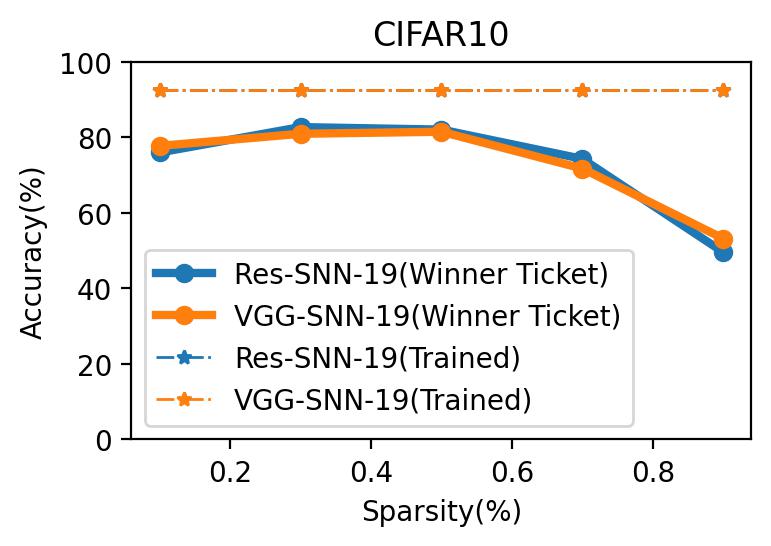}}
  
        \subfigure[CIFAR100]{
		\includegraphics[width=\linewidth]{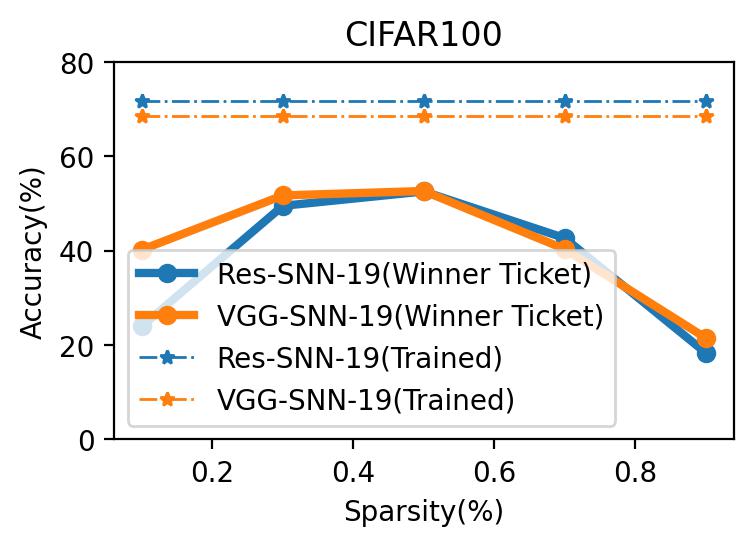}}
    \vspace{-4mm}
	\caption{The horizontal and vertical coordinates represent the sparsity and accuracy. The solid line is the untrained sub-network. The dashed line is the trained large network (sparsity=$0\%$).}
 \label{fig:subnetwork_scearch}
\end{wrapfigure}

We apply the \texttt{edge-popup} in SNN-VGG-16 and Res-SNN-19. We set maskable weights for all layers (the first layer and the last layer are included). The hyper-parameter sparsity $k\%$ is set every 0.2 from 0.1 to 0.9. The datasets used to verify are CIFAR10/100. The network structures and other experiment details are shown in \Cref{app_sec:exp_details}. As shown in \Cref{fig:subnetwork_scearch}, the existence of good sub-networks is empirically verified, which is highly conformed with our theoretical analysis. The additional timesteps and discontinuous activation complicate the sub-network search problem in SNNs, but good sub-networks still exist without any weight training and can be found by specific algorithms.



\section{Pruning Criterion Based on Probabilistic Modeling}\label{sec:algorithm}

\textbf{Discussion.} For two spiking neurons, layers, or networks that have the same input but different weights, the proposed probabilistic modeling method can give the probability that the outputs (that is, the spiking firing) of both are the same. We convert the perturbation brought by pruning under the same input into the error of spatial input features, and then analyze its impact on spiking firing. Then, we prove that the lottery ticket hypothesis also holds in SNNs.

\begin{wrapfigure}[16]{r}{0.4\textwidth}
\vspace{4mm}
\begin{center}
\centerline{\includegraphics[width=\linewidth]{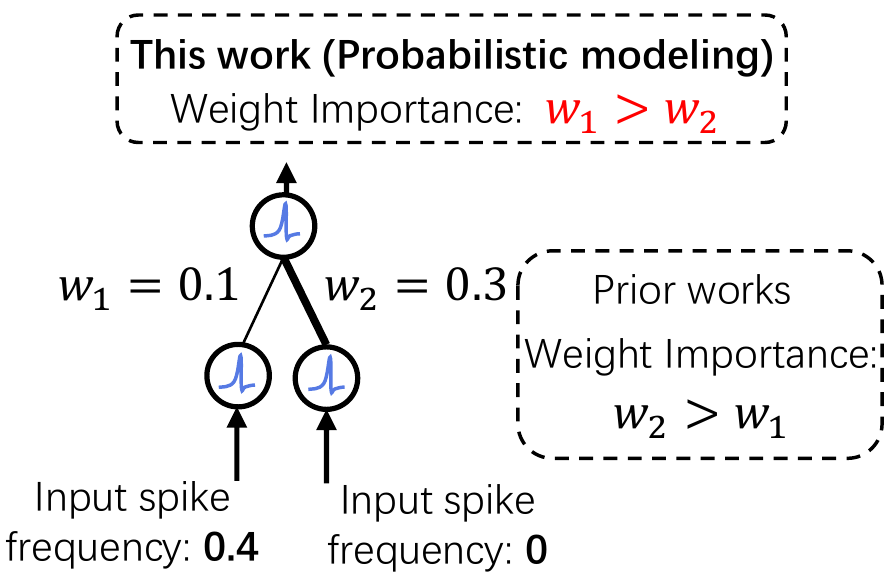}}
\vspace{-4mm}
\caption{The importance of weight needs to consider both the magnitude of the weight and the firing of neurons.}
\label{fig:weight_importance}
\end{center}
\end{wrapfigure}

The potential of probabilistic modeling goes far beyond that. Probabilistic modeling prompts us to rethink whether the existing pruning methods in SNNs are reasonable, since most methods directly continue the idea and methods of ANN pruning. Firstly, from the view of linear transformation, \emph{\textbf{how to metrics the importance of weights in SNNs?}} The relationship between the inner state of the artificial neuron before activation and the input is usually not considered in ANN pruning. But we have to take this into account in binary-activated SNN pruning because the membrane potential is related to both weights and inputs. For instance, when the input spike frequency is 0, no matter how large the weight is, it will not affect the result of the linear transformation (\Cref{fig:weight_importance}). Secondly, from the view of nonlinear transformation, \emph{\textbf{how does pruning affect the output of SNNs?}} As shown in \Cref{fig:abstract}\textbf{b}, when the membrane potential is in different intervals, the probability that the output of the spiking neuron is changed is also different. Thus, we study this important question theoretically using probabilistic modeling.

Moreover, we can assume that the weights of the network remain unchanged, and add disturbance to the input, then exploit probabilistic modeling to analyze the impact of input perturbations on the output, i.e., analyze the robustness \cite{kundu2021hire} of the SNNs in a new way. Similarly, probabilistic modeling can also analyze other compression methods in SNNs, such as quantization \cite{deng2021comprehensive}, tensor decomposition \cite{KCP_SNN}, etc. The perturbations brought about by these compression methods can be converted into errors in the membrane potential, which in turn affects the firing of spiking neurons.

\textbf{Pruning criterion design.} We design a new criterion to prune weights according to their probabilities of affecting the firing of spiking neurons. Based on \Cref{theorem:theorem_in_chapter_3}, for each weight, its influence on the output of the spiking neuron has a probability upperbound $P$, which can be estimated as:
\begin{equation}
P\approx E{(\left|u^{\prime} - u \right|)}\mathcal{N}(0|\mu-u_{th}, {var}),
\label{eq:pro}
\end{equation}
where $E{(\left|u^{\prime} - u \right|)}$ is the expectation of the error in the membrane potential brought about by pruning. In this work, it is written as:
\begin{equation}
E{(\left|u^{\prime} - u \right|)} = \frac{E_{act}[\left|w\right|]\left|\gamma\right|}{\sqrt{\sigma_{\mathcal{B}}^{2}+\epsilon}},
\label{eq:error_experctation}
\end{equation}

where $E_{act}[w]$ is the expectation that a weight is activated (the pre-synaptic spiking neuron outputs a spike). $\gamma$ and $\sigma_{B}$ are a pair of hyper-parameters in Batch Normalization \cite{ioffe2015batch}, which can also affect the linear transformation of weights, thus we incorporate them into \Cref{eq:error_experctation}. The detailed derivation of \Cref{eq:error_experctation} can be found in \Cref{app_sec:pruning_method}. Note, in keeping with the notations in classic BN, there is some notation mixed here, but only for \Cref{eq:pro} and \Cref{eq:error_experctation}. Next, following \cite{zheng2021going,guo2022recdis}, here we suppose the membrane potential follows Gaussian distribution $\mathcal{N}(\mu, {var})$. Consequently, $\mathcal{N}(0|\mu-u_{th}, {var})$ represents the probability density when the membrane potential is at the threshold $u_{th}$. Finally, for each weight, we have a probability $P$, and we prune according to the size of $P$.


\begin{wrapfigure}[24]{r}{0.4\textwidth}
	\centering  
    \vspace{-6mm}
	\subfigbottomskip=2pt 
	\subfigcapskip=-5pt 
	\subfigure[CIFAR10]{
		\includegraphics[width=\linewidth]{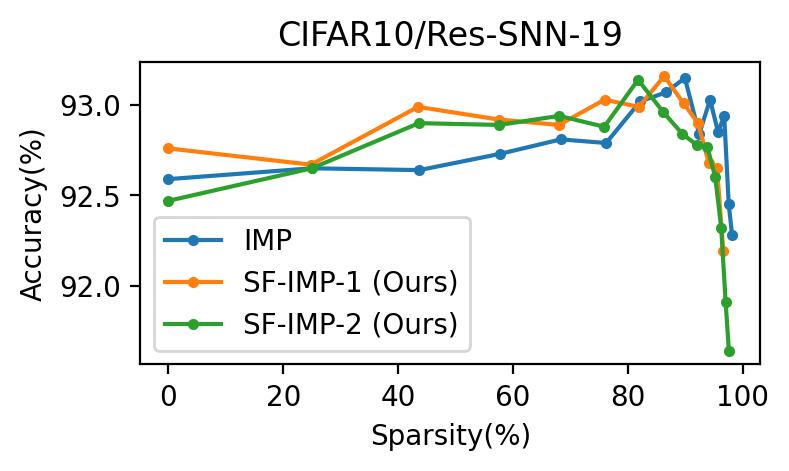}}
  
        \subfigure[CIFAR100]{
		\includegraphics[width=\linewidth]{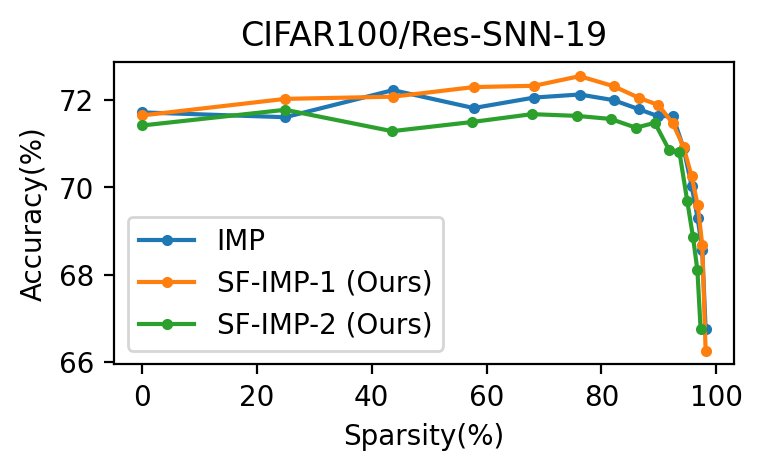}}
	\caption{Observations from Iterative Magnitude Pruning (IMP) with the proposed pruning criterion $P$ in \cref{eq:pro}.}
 \label{fig:exper_result}
\end{wrapfigure}

\textbf{Experiments.} Since our subject is to theoretically prove LTH in SNNs, we test the proposed new criteria using methods related to LTH. Note, \Cref{eq:pro} is a general pruning criterion and can also be incorporated in traditional non-LTH SNN pruning methods. Specifically, in this work we employ the LTH-based Iterative Magnitude Pruning (IMP) method for prune \cite{frankle2018the}, whose criterion is to prune weights with small absolute values (algorithm details are given in \Cref{app_sec:algorithm}). Then, we change the criterion to pruning according to the size of $P$ in \Cref{eq:pro}.


We re-implement the pipeline network (Res-SNN-19 proposed in \cite{zheng2021going}) and IMP pruning method on the CIFAR-10/100 \cite{krizhevsky2009learning} datasets. We use the official code provided by authors in \cite{LTH_SNN_2022} for the whole pruning process. Then, we perform rigorous ablation experiments without bells and whistles, i.e., all experiment settings are same, the only change is to regulate the pruning criterion to \Cref{eq:pro}. This is a IMP pruning based on Spiking Firing, thus we name it SF-IMP-1 (\Cref{app_sec:algorithm}). Moreover, the encoding layer (the first convolutional layer) in SNNs has a greater impact on task performance, and we can also allow the weights pruning in the encoding layer to be reactivated \cite{guo2016dynamic}, called SF-IMP-2 (\Cref{app_sec:algorithm}).

The experimental results are shown in \Cref{fig:exper_result}. First, these two approaches perform marginally differently on CIFAR-10/100, with SF-IMP-2 outperforming on CIFAR-10, and SF-IMP-1 performing better on CIFAR-100. Especially, compared with the vanilla IMP algorithm on CIFAR-100, SF-IMP-1 performs very well at all pruning sparsities. But on CIFAR-10, when the sparsity is high, the vanilla IMP looks better. One of the possible reasons is that the design of \Cref{eq:pro} still has room for improvement. These preliminary results demonstrate the effectiveness of the pruning criteria designed based on probabilistic modeling, and also lay the foundation for the subsequent design of more advanced pruning algorithms.

\section{Conclusions}
In this work, we aim to theoretically prove that the Lottery Ticket Hypothesis (LTH) also holds in SNNs, where the LTH theory has had a wide influence in traditional ANNs and is often used for pruning. In particular, spiking neurons employ binary discrete activation functions and have complex spatio-temporal dynamics, which are different from traditional ANNs. To address these challenges, we propose a new probabilistic modeling method for SNNs, modeling the impact of pruning on spiking firing, and then prove both theoretically and experimentally that LTH\footnote{Specifically and strictly, strong LTH. We get a sub-network that performs well without any training.} holds in SNNs. We design a new pruning criterion from the probabilistic modeling view and test it in an LTH-based pruning method with promising results. Both the probabilistic modeling and the pruning criterion are general. The former can also be used to analyze the robustness and other network compression methods of SNNs, and the latter can also be exploited in non-LTH pruning methods. In conclusion, we have for the first time theoretically established the link between membrane potential perturbations and spiking firing through the proposed novel probabilistic modeling approach, and we believe that this work will provide new inspiration for the field of SNNs.

\bibliographystyle{unsrt} 
\bibliography{neurips_2023}

\newpage
\appendix
\setcounter{equation}{0}
\renewcommand\theequation{S.\arabic{equation}} 
\onecolumn
\section{Notations}\label{app_sec:notation}


\begin{table*}[!ht]
\caption{Important Notations. }
\label{sample-table}
\vskip 0.15in
\begin{center}
\begin{small}
\begin{tabular}{cc}
\toprule 
Symbol  & Definition \\
\midrule
$t$    & Timestep index \\
$l$    & Layer index \\
$i$    & Neuron index \\
$u_{i}^{t, l}$ & Membrane potential of spiking neuron \\
$h_{i}^{t, l}$ & Temporal input of spiking neuron  \\
$s_{i}^{t, l}$ & Spatial input of spiking neuron   \\
$x_{i}^{t, l}$ & Spatial input feature of spiking neuron  \\
$u_{th}$       & Firing threshold\\
$\operatorname{Hea}(\cdot)$ & Heaviside step function \\
$V_{reset}$    & Reset membrane potential after firing a spike\\
$\beta$       & Decay factor\\
$w^{l}_{ij}$       & Weight connect from two spiking neuron\\
$u_{th}$       & Firing threshold\\

$x$ & Spatial input feature of a spiking neuron without indexes\\
$x^{\prime}$ & Spatial input feature after adding perturbation\\
$u$ & Membrane potential of a spiking neuron without indexes\\
$u^{\prime}$ & Membrane potential after adding perturbation\\

$\sigma_{x^{1:t-1}}^{t}$ & Membrane potential depends on all previous spatial input features\\
$\sigma$ & Shorthand for activation function $\sigma_{x^{1:t-1}}^{t}$\\

$N$ & Width of network\\
$T$ & Total timestep\\

$G$ & Large network\\
$G^{l}$ & $l$-th layer of large network\\
$\hat{G}$ & Target network\\
$\tilde{G}$ & Equivalent network\\

$\boldsymbol{s}^{t}$ & Spatial input vector at timestep $t$\\
$\boldsymbol{s}$ & Spatial input vector without timestep index\\
$\boldsymbol{S}^{t}$ & Spatial input tensor at timestep $t$\\
$\boldsymbol{S}^{1:t}$ & All spatial input tensors before timestep $t$\\
$\boldsymbol{S}$ & Shorthand for $\boldsymbol{S}^{1:t}$\\
$\mathcal{S}$ & Dataset \\

$\boldsymbol{v}$ & Virtual layer weight\\
$\boldsymbol{V}$ & Weight matrix of Virtual layer\\
$\boldsymbol{b}$ & Mask of weight vector\\
$\boldsymbol{B}$ & Mask of weight matrix\\
$\boldsymbol{W}^{l}$ & Weight matrix of $l$-th layer\\
$C_{th}$ & A constant related to the hyper-parameter $u_{th}$\\
$C$ & A constant related to the dataset $\mathcal{S}$ and threshold $u_{th}$\\
$\mathbb{E}$ & Event\\
$k$ & Network width required for LTH establishment\\

\bottomrule
\end{tabular}
\end{small}
\end{center}
\label{table:notation}
\vskip -0.1in
\end{table*}

  



\section{Proofs of Probabilistic Modeling}\label{app_sec:pro_modeling}
\begin{lemma}
\label{lemma:upper bound_appendix}
For the probability density function $p(\cdot)$, if it is v-Neighborhood-Finite Distribution, for any $\delta \skx{>0}$, there exists $\epsilon \skx{>0} $ that$\int_{v-\epsilon}^{v+\epsilon}p(x)\mathrm{dx} \leq \delta$.
\end{lemma}
\begin{proof}
Since distribution $p$ is v-Neighborhood-Finite Distribution, for any $\delta \skx{>0}$,\skx{there} exists \skx{a constant} $\epsilon \skx{>0}$so that $x \in (v-\epsilon, v+\epsilon)$
\begin{equation}
p_{sup} = \mathop{sup}\limits_{x}p(x) \geq p(x),
\end{equation}
thus:
\begin{equation}
\int_{v-\epsilon}^{v+\epsilon}p(x)\mathrm{dx} \leq 2\epsilon p_{sup},
\end{equation}
when $\epsilon \leq \frac{p_{sup}}{2 \delta}$, the lemma holds.

\end{proof}

\begin{lemma}
\label{lemma:singleEFR_appendix}
Suppose the spiking neurons are $u_{th}$-Neighborhood-Finite Distribution and the inputs of two corresponding spiking neuron got an error upperbound $\epsilon$, and they got the same inner state $h^{T-1}$, the probability upperbound of different outputs is proportional to $\epsilon$.

For norm, we use $\left\| \cdot \right\|_{0}$ and $\left\| \cdot \right\|_{2}$. The $\left\| \cdot \right\|_{0}$ is used to count the number of non-zero weights while the $\left\| \cdot \right\|_{2}$ is used to measure the distance of outputs.

Formally: 

For two spiking neurons $\hat{\sigma}^{T}$ and $\tilde{\sigma}^{T}$, when $\tilde{h}^{T-1}=\hat{h}^{T-1}$ and $\hat{u}^{T}=\hat{h}^{T-1} + \hat{x}^{T}$ is a random variable follows the $u_{th}$-Neighborhood-Finite Distribution, if $\|\tilde{x}^{T} - \hat{x}^{T}\|\leq \epsilon$, then:
$P\left[\hat{\sigma}^{T}(\hat{x}^{T})\neq\tilde{\sigma}^{T}(\tilde{x}^{T}) \right]\propto \epsilon $
\end{lemma}

\begin{proof}
Since distribution of $\hat{u}^{T}=\hat{h}^{T-1} + \hat{x}^{T}$, denotes as $p$, is $u_{th}$-Neighborhood-Finite Distribution, for any $\delta$, exists $\epsilon$ that $\hat{u} \in (u_{th}-\epsilon, u_{th}+\epsilon)$
\begin{equation}
p_{sup} = \mathop{sup}\limits_{\hat{u}^{T}}p(\hat{u}^{T}) \geq p(\hat{u}^{T}),
\end{equation}
since $\tilde{u}^{T} = \tilde{h}^{T-1}+\tilde{x}^{T} = \hat{u}^{T} + \tilde{x}^{T} - \hat{x}^{T}$, $\|\hat{u}^{T}-\tilde{u}^{T}\| \leq \epsilon,$
then:
\begin{equation}
\begin{aligned}
& \quad P\left[\hat{\sigma}^{T}(\hat{x}^{T})\neq\tilde{\sigma}^{T}(\tilde{x}^{T}) \right] \\
&= P\left[sign(\hat{u}^{T}-u_{th})sign(\tilde{u}^{T}-u_{th})=-1\right] \\
&\leq \int_{u_{th}-\epsilon}^{u_{th}+\epsilon} p(\hat{u}^{T})\mathrm{d}u \leq 2p_{sup}\epsilon,
\end{aligned}
\end{equation}
thus the lemma holds.
\end{proof}

Note, for timestep $T$, the membrane potential will follow $p(u^{T}; S^{1:T}, W)$, where $S^{1:T}\in\{0, 1\}^{N\times T}$ is the input from timesteps 1 to $T$, $W \in \mathbf{R}^{N}$ is the weight vector for the linear combination that shared at each timestep. When we do not emphasize the specific timestep and input data, we use $p(u)$ for brevity. 

\begin{lemma}
\label{lemma:multiTEFR_appendix}
Suppose the spiking neurons are $u_{th}$-Neighborhood-Finite Distribution at timestep $T$ and the inputs of two corresponding spiking neurons got an error upperbound $\epsilon$ at any timestep, and they got the same inner state $h^{0}$, if there is no different output at the first $T-1$ timesteps, then the probability upperbound is proportional to $\frac{\epsilon}{1 - \beta}$.

Formally: 

For two spiking neurons $\hat{\sigma}^{T}$ and $\tilde{\sigma}^{T}$, when $\tilde{h}^{0}=\hat{h}^{0}$ and $\hat{u}^{T}=\hat{h}^{T-1} + \hat{x}^{T}$ is a random variable follows the $u_{th}$-Neighborhood-Finite Distribution, if $\|\tilde{x}^{t} - \hat{x}^{t}\|\leq \epsilon$ and $\hat{\sigma}^{t}(\hat{x}^{t})=\hat{\sigma}^{t}(\tilde{x}^{t})$ for $t = 1, 2, \cdots, T-1$, then:
$P\left[\hat{\sigma}^{T}(\hat{x}^{T})\neq\tilde{\sigma}^{T}(\tilde{x}^{T}) \right]\propto \frac{\epsilon}{1 - \beta}$. 

\end{lemma}

\begin{proof}
For each timestep $i$, we got:

\begin{equation}
\hat{u}^{t}=\hat{h}^{t-1} + \hat{x}^{t}
\end{equation}
\begin{equation}
\tilde{u}^{t}=\tilde{h}^{t-1} + \tilde{x}^{t},
\end{equation}

at timestep $t=1$, we have:
\begin{equation}
\| \hat{u}^{1} = \tilde{u}^{1}\| \leq \epsilon,
\end{equation}
if the same output of the former timestep $t-1$ is $1$, there will be no error for inner state thus $\hat{h}^{t-1}=\tilde{h}^{t-1}$, then $\|\hat{u}^{t}-\tilde{u}^{t}\|\leq \epsilon$, otherwise, the error will be $\|\hat{u}^{t}-\tilde{u}^{t}\|\leq \epsilon + \beta \|\hat{u}^{t-1}-\tilde{u}^{t-1}\|$.

Thus, by iterating, here is the upperbound error for membrane potential in the timestep $T$:
\begin{equation}
\|\hat{u}^{T}-\tilde{u}^{T}\|\leq \sum_{n=0}^{T-1}\beta^{n}\epsilon \leq \sum_{n=0}^{+\infty}\beta^{n}\epsilon \leq \frac{\epsilon}{1-\beta}
\end{equation}
then:
\begin{equation}
\quad P\left[\hat{\sigma}^{T}(\hat{x}^{T})\neq\tilde{\sigma}^{T}(\tilde{x}^{T}) \right] \leq 2p_{sup}\frac{\epsilon}{1-\beta},
\end{equation}
where $p_{sup} = \mathop{sup}\limits_{\hat{u}^{T}}p(\hat{u}^{T}) \geq p(\hat{u}^{T})$.

Here the lemma holds.
\end{proof}

\begin{theorem}
Suppose the spiking layers are $u_{th}$-Neighborhood-Finite Distribution at timestep $T$ and the inputs of two corresponding spiking layers with a width $N$ got an error upperbound $\epsilon$ for each element of input vectors at any timestep, and they got the same inner state vector $\boldsymbol{h^{0}}$, if there is no different output at the first $T-1$ timesteps, then the probability upperbound is proportional to ${N}\frac{\epsilon}{1 - \beta}$.

Formally: 

For two spiking layers $\hat{\sigma}^{T}$ and $\tilde{\sigma}^{T}$, when $\boldsymbol{\tilde{h}^{0}}=\boldsymbol{\hat{h}^{0}}$ and $\boldsymbol{\hat{u}^{T}}=\boldsymbol{\hat{h}^{T-1}} + \boldsymbol{\hat{x}^{T}}$ is a random variable follows the $u_{th}$-Neighborhood-Finite Distribution, if $\|\tilde{x}^{t}_{k} - \hat{x}^{t}_{k}\|\leq \epsilon$ ($k=1, 2, \cdots, N; t=1, 2, \cdots, T$) and $\hat{\sigma}^{t}(\boldsymbol{\hat{x}^{t}})=\hat{\sigma}^{t}(\boldsymbol{\tilde{x}^{t}})$ for $t = 1, 2, \cdots, T-1$, then:
$P\left[\hat{\sigma}^{T}(\boldsymbol{\hat{x}^{T}})\neq\tilde{\sigma}^{T}(\boldsymbol{\tilde{x}^{T}}) \right]\propto {N}\frac{\epsilon}{1 - \beta}$. 
\end{theorem}
\begin{proof}
Here, we define $p_{sup}$ as the upperbound probability density of all entries at timestep $T$. Thus:
\begin{equation}
p_{sup}=\mathop{sup}\limits_{\hat{u}^{T}_{k}}p(\hat{u}^{T}_{k}),
\end{equation}
then, according to \Cref{lemma:multiTEFR}, for any single entry we have:
\begin{equation}
P\left[\hat{\sigma}^{T}_{k}(\hat{x}^{T}_{k})\neq\tilde{\sigma}^{T}_{k}(\tilde{x}^{T}_{k})\right] \leq \frac{2}{1-\beta}p_{sup}\epsilon
\end{equation}
then, according to the union bound inequality:
\begin{equation}
P\left[\exists k, \hat{\sigma}^{T}_{k}(\hat{x}^{T}_{k})\neq\tilde{\sigma}^{T}_{k}(\tilde{x}^{T}_{k})\right] \leq N\frac{2}{1-\beta}p_{sup}\epsilon
\end{equation}
\end{proof}

\section{Proofs of Lottery Ticket Hypothesis in SNNs}\label{app_sec:proof_LTH}
\begin{lemma}
\label{lemma:SWA_appendix}
\textbf{Single Weight Approximation.} Fix the weight scalar $\hat{w} \in [-\frac{1}{\sqrt{N}}, \frac{1}{\sqrt{N}}]$ which is the connection in target network between two neurons. The equivalent structure is $k$ spiking neurons with $k$ weights $\boldsymbol{v}=[v_{1}, v_{2}, \cdots, v_{k}]^\mathsf {T}$ connect the input and $\boldsymbol{\tilde{w}}=[\tilde{w}_{1}, \tilde{w}_{2}, \cdots, \tilde{w}_{k}]^\mathsf {T}$ connect out. All the weights $\tilde{w}_{i}$ and $v_{i}$ random initialized with uniform distribution $\mathrm{U}[-1, 1]$ and i.i.d. $\boldsymbol{b}=[b_{1}, b_{2}, \cdots, b_{k}]$ is the mask for weight vector $\boldsymbol{v}$, and $\boldsymbol{b}\in \{0, 1\}^{k}, \left\|\boldsymbol{b}\right\|_{0} \leq 1$. Then, let the function of equivalent structure be $\tilde{g}(s)=(\boldsymbol{\tilde{w}}\odot \boldsymbol{b})^\mathsf {T}\tilde{\sigma}(\boldsymbol{v}s)$, where input spiking $s$ is a scalar that $s\in\{0, 1\}$. Then, $\forall \skx{0 <}\delta \skx{\leq 1} , \exists \epsilon \skx{>0}$ when $$k \geq \frac{1}{C_{th}\epsilon}\
\log{\frac{1}{\delta}}$$ and $C_{th}=\frac{1 - u_{th}}{2}$, there exists a mask vector $b$ that $$\left\|\tilde{g}(s)-\hat{w}s\right\| \leq\epsilon$$ w.p at least $1-\delta$.
\end{lemma}

\begin{proof}
For $k$ spiking neurons in equivalent structure, since $\left\|b\right\|_{0} \leq 1$, only one spiking neuron is active while others are pruned out with their weights. For the chosen active neuron, the weights $\tilde{w}_{i}$ and $v_{i}$ should satisfy the following sufficient conditions:
\begin{itemize}
    \item $\left\|\tilde{w}_{i}-\hat{w}_{i}\right\| \leq \epsilon$
    \item $v_{i} \geq v_{th}$
\end{itemize}
Since $\tilde{w}, v \sim \mathrm{U}[-1, 1]$,
\begin{equation}
\mathrm{P}\left[\left\|\tilde{w}_{i}-\hat{w}_{i}\right\| \leq \epsilon\right] = \frac{2\epsilon}{2} = \epsilon
\end{equation}
\begin{equation}
\mathrm{P}\left[ v_{i} \geq v_{th}\right] = \frac{1 - v_{th}}{2} = C_{th}
\end{equation}
$C_{th}$ is the constant that $C_{th} = \frac{1 - v_{th}}{2} $.

The probability for a single spiking neuron that satisfies the condition is:
\begin{equation}
\begin{aligned}
&\mathrm{P}\left[\left\|\tilde{w}_{i}-\hat{w}_{i}\right\| \leq \epsilon \land v_{i} \geq v_{th}\right] \\
= &\mathrm{P}\left[\left\|\tilde{w}_{i}-\hat{w}_{i}\right\| \leq \epsilon\right]\mathrm{P}\left[ v_{i} \geq v_{th}\right] \\ 
= &C_{th}\epsilon
\end{aligned}
\end{equation}
The probability of not satisfying the condition for all $k$ spiking neurons is:
\begin{equation}
\begin{aligned}
&\mathrm{P}\left[\forall i \in [k], \neg (\left\|\tilde{w}_{i}-\hat{w}_{i}\right\| \leq \epsilon \land v_{i} \geq v_{th})\right] \\
= & (1-C_{th}\epsilon)^{k} \leq \exp{( - k C_{th} \epsilon)} \leq \delta
\end{aligned}
\end{equation}
Thus, when $k \geq \frac{1}{C_{th}\epsilon}\
\log{\frac{1}{\delta}}$, the active spiking neuron satisfies the condition with probability at least $1-\delta$.
\end{proof}

\begin{lemma}
\label{lemma:LWA_appendix}
\textbf{Layer Weights Approximation.} Fix the weights vector $\boldsymbol{\hat{w}} \in [-\frac{1}{\sqrt{N}}, \frac{1}{\sqrt{N}}]^{N}$ which is the connection in target network between a layer of spiking inputs and a neuron. The equivalent structure is $k$ spiking neurons with $k \times N$ weights $\boldsymbol{V}\in\mathbf{R}^{k\times N}$ connect the input and $\boldsymbol{\tilde{w}}=[\tilde{w}_{1}, \tilde{w}_{2}, \cdots, \tilde{w}_{k}]^\mathsf{T}$ connect out. All the weights $\tilde{w}_{i}$ and $v_{ij}$ random initialized with uniform distribution $\mathrm{U}[-1, 1]$ and i.i.d. $\boldsymbol{b}\in\{0, 1\}^{k}$ is the mask for matrix $\boldsymbol{\tilde{w}}$, $\left\|\boldsymbol{b}\right\|_{0} \leq N$. Then, let the function of equivalent structure be $\tilde{g}(\boldsymbol{s})=(\boldsymbol{\tilde{w}}\odot\boldsymbol{b})^{\mathsf{T}}\tilde{\sigma}(V \boldsymbol{s})$, where input spiking $\boldsymbol{s}$ is a vector that $\boldsymbol{s}\in\{0, 1\}^{N}$. Then, $\forall \skx{0<}\delta \skx{\leq 1}, \exists \epsilon \skx{>0}$ when $$k \geq N\lceil\frac{N}{C_{th}\epsilon}\log{\frac{N}{\delta}}\rceil$$ and $C_{th}=\frac{1 - u_{th}}{2}$, there exists a mask vector $\boldsymbol{b}$ that $$\left\|\tilde{g}(\boldsymbol{s})-\boldsymbol{\hat{w}}^{\mathsf{T}}\boldsymbol{s}\right\| \leq\epsilon$$ w.p at least $1-\delta$.
\end{lemma}

\begin{proof}
Define the event:
\begin{equation}
\begin{aligned}
\mathbb{E}_{i, k', \epsilon}=\{\forall a \in [k'], \neg (\left\|\tilde{w}_{(i-1)k'+a}-\hat{w}_{i}\right\| \leq \epsilon \land v_{(i-1)k'+a, i} \geq u_{th})\},
\end{aligned}
\end{equation}
this event means in a block of approximation structures with $k'$ structures to approximate the connecting weight from the $i$-th input to the output, but no one structure satisfies the $\epsilon$ approximation error condition.

Thus 
\begin{equation}
P(\mathbb{E}_{i, k', \frac{\epsilon}{N}})\leq \exp(-k' C_{th} \frac{\epsilon}{N}),
\end{equation}
We totally have $N$ blocks thus $k = N k'$ and for each block, using union bound inequality, we have:
\begin{equation}
P(\bigcup_{i}\mathbb{E}_{i, k', \frac{\epsilon}{N}})\leq \sum\limits_{i}P(\mathbb{E}_{i, k', \frac{\epsilon}{N}}) \leq N\exp({-k' C_{th} \frac{\epsilon}{N}}) \leq \delta,
\end{equation}
thus:
\begin{equation}
P((\bigcup_{i}\mathbb{E}_{i, k', \frac{\epsilon}{N}})^\mathsf{C}) \geq 1 - \delta,
\end{equation}
where
\begin{equation}
k \geq N\lceil\frac{N}{C_{th}\epsilon}\log{\frac{N}{\delta}}\rceil
\end{equation}
then the lemma holds.
\end{proof}

\begin{lemma}
\label{lemma:LLA_appendix}
\textbf{Layer to Layer Approximation.} Fix the weight matrix $\boldsymbol{\hat{W}} \in [-\frac{1}{\sqrt{N}}, \frac{1}{\sqrt{N}}]^{N \times N}$ which is the connection in target network between a layer of spiking inputs and the next layer of neurons. The equivalent structure is $k$ spiking neurons with $k \times N$ weights $\boldsymbol{V}\in\mathbf{R}^{k\times N}$ connect the input and $\boldsymbol{\tilde{W}}\in\mathbf{R}^{N \times k}$ connect out. All the weights $\tilde{w}_{ij}$ and $v_{ij}$ random initialized with uniform distribution $\mathrm{U}[-1, 1]$ and i.i.d. $\boldsymbol{B}\in\{0, 1\}^{N \times k}$ is the mask for matrix $\boldsymbol{\tilde{W}}$, $\sum_{i, j}\left\|B_{ij}\right\|_{0} \leq N^2, \sum_{j}\left\|B_{ij}\right\|_{0} \leq N$. Then, let the function of equivalent structure be $\tilde{g}(\boldsymbol{s})=(\boldsymbol{\tilde{W}}\odot\boldsymbol{B})\tilde{\sigma}(\boldsymbol{V}\boldsymbol{s})$, where input spiking $\boldsymbol{s}$ is a vector that $\boldsymbol{s}\in\{0, 1\}^{N}$. Then, $\forall \skx{0<} \delta \skx{\leq 1}, \exists \epsilon \skx{>0}$ when $$k \geq N\lceil\frac{N}{C_{th}\epsilon}\log{\frac{N\skx{^{2}}}{\delta}}\rceil$$ and $C_{th}=\frac{1 - u_{th}}{2}$, there exists a mask matrix $\boldsymbol{B}$ that $$\left\|[\tilde{g}(\boldsymbol{s})]_{i}-[\hat{W}\boldsymbol{s}]_{i}\right\| \leq\epsilon$$ w.p at least $1-\delta$.
\end{lemma}
\begin{proof}
Define the event:
\begin{equation}
\mathbb{E}_{j, i, k', \epsilon}=\{\forall a \in [k'], \neg (\left\|\tilde{w}_{j, (i-1)k'+a}-\hat{w}_{ji}\right\| \leq \epsilon \land v_{(i-1)k'+a, i} \geq u_{th})\},
\end{equation}
this event means in a block of approximation structures with $k'$ structures to approximate the connecting weight from the $i$-th input to the $j$-th output, but no one structure satisfies the $\epsilon$ approximation error condition.

Thus 
\begin{equation}
P(\mathbb{E}_{j, i, k', \frac{\epsilon}{N}})\leq \exp({-k' C_{th} \frac{\epsilon}{N}}),
\end{equation}
We totally have $N$ blocks thus $k = N k'$ and for each block, using union bound inequality, we have:
\begin{equation}
P(\bigcup_{j, i}\mathbb{E}_{j, i, k', \frac{\epsilon}{N}})\leq \sum\limits_{j, i}P(\mathbb{E}_{j, i, k', \frac{\epsilon}{N}}) \leq N^{2}\exp({-k' C_{th} \frac{\epsilon}{N}}) \leq \delta,
\end{equation}
thus:
\begin{equation}
P((\bigcup_{j, i}\mathbb{E}_{j, i, k', \frac{\epsilon}{N}})^\mathsf{C}) \geq 1 - \delta,
\end{equation}
where
\begin{equation}
k \geq N\lceil\frac{N}{C_{th}\epsilon}\log{\frac{N^{2}}{\delta}}\rceil
\end{equation}
then the lemma holds.
\end{proof}

\begin{lemma}
\label{lemma:LAA_appendix}
\textbf{Layer Spiking Activation Approximation.} Fix the weight matrix $\boldsymbol{\hat{W}} \in [-\frac{1}{\sqrt{N}}, \frac{1}{\sqrt{N}}]^{N \times N}$ which is the connection in target network between a layer of spiking inputs and the next layer of neurons. The equivalent structure is $k$ spiking neurons with $k \times N$ weights $\boldsymbol{V}\in\mathbf{R}^{k\times N}$ connect the input and $\boldsymbol{\tilde{W}}\in\mathbf{R}^{N \times k}$ connect out. All the weights $\tilde{w}_{ij}$ and $v_{ij}$ random initialized with uniform distribution $\mathrm{U}[-1, 1]$ and i.i.d. $\boldsymbol{B}\in\{0, 1\}^{N \times k}$ is the mask for matrix $\boldsymbol{V}$, $\sum_{i, j}\left\|B_{ij}\right\|_{0} \leq N^2, \sum_{j}\left\|B_{ij}\right\|_{0} \leq N$. Then, let the function of equivalent structure be $\tilde{g}(\boldsymbol{s})=\tilde{\sigma}((\boldsymbol{\tilde{W}}\odot \boldsymbol{B})\tilde{\sigma}(\boldsymbol{V}\boldsymbol{s}))$, where input spiking $\boldsymbol{s}$ is a vector that $\boldsymbol{s}\in\{0, 1\}^{N}$. $C$ is the constant depending on the supremum probability density of the dataset of the network. Then, $\forall \skx{0<}\delta \skx{\leq 1}, \exists \epsilon \skx{>0}$ when $$k \geq N^2\lceil\frac{N}{C_{th}\epsilon}\log{\frac{N^2}{\delta-NC\epsilon}}\rceil$$, there exists a mask matrix $\boldsymbol{B}$ that $$\left\|\tilde{g}(\boldsymbol{s})-\hat{\sigma}(\boldsymbol{\hat{W}}\boldsymbol{s})\right\| = 0$$ w.p at least $1-\delta$.
\end{lemma}

\begin{proof}
The definition of $\mathbb{E}_{j, i, k', \epsilon}$ follows the proof of Lemma\Cref{lemma:LLA_appendix}, thus we have:
\begin{equation}
P(\bigcup_{j, i}\mathbb{E}_{j, i, k', \frac{\epsilon}{N}})\leq \sum\limits_{j, i}P(\mathbb{E}_{j, i, k', \frac{\epsilon}{N}}) \leq N^{2}\exp({k' C_{th} \frac{\epsilon}{N}})
\end{equation}

And the event $(\bigcup_{j, i}\mathbb{E}_{j, i, k', \frac{\epsilon}{N}})^\mathsf{C}$ implies that the error of each channel smaller than $\epsilon$.

According to the proof of theorem\Cref{theorem:theorem_in_chapter_3}, the probability upperbound of different output of spiking neurons with the same temporal state at $t=0$ is $2Np_{sup}\frac{\epsilon}{1-\beta}$. We define $\mathbb{E}_{fire}$ as the event of different output of corresponding spiking layer.

Then, according to union bound, we have:

\begin{equation}
\begin{aligned}
&P((\bigcup_{j, i}\mathbb{E}_{j, i, k', \frac{\epsilon}{N}})\bigcup\mathbb{E}_{fire})\\ \leq &\sum\limits_{j, i}P(\mathbb{E}_{j, i, k', \frac{\epsilon}{N}})+P(\mathbb{E}_{fire}) \leq N^{2}\exp({-k' C_{th} \frac{\epsilon}{N}}) + Np_{sup}\frac{2\epsilon}{1-\beta} \leq \delta,
\end{aligned}
\end{equation}
 Let $C = 2p_{sup}\frac{1}{1-\beta}$, then, we have:
 \begin{equation}
k \geq N^2\lceil\frac{N}{C_{th}\epsilon}\log{\frac{N^2}{\delta-NC\epsilon}}\rceil
 \end{equation}
 Here, the event $((\bigcup_{j, i}\mathbb{E}_{j, i, k', \frac{\epsilon}{N}})\bigcup\mathbb{E}_{fire})^\mathsf{C}$ implies that the error of each entry is smaller than $\epsilon$ while the output have no difference.
 \begin{equation}
 P(((\bigcup_{j, i}\mathbb{E}_{j, i, k', \frac{\epsilon}{N}})\bigcup\mathbb{E}_{fire})^\mathsf{C}) \geq 1 - \delta,
 \end{equation}
 the lemma holds.
\end{proof}

\begin{lemma}
\label{lemma:ALA_appendix}
\textbf{All Layers Approximation.} Fix the weight matrix $\boldsymbol{\hat{W}}^{l} \in [-\frac{1}{\sqrt{N}}, \frac{1}{\sqrt{N}}]^{N \times N}$ which is the connection in target network between a layer of spiking inputs and the next layer of neurons. The equivalent structure is $k$ spiking neurons with $k \times N$ weights $\boldsymbol{V}^{l}\in\mathbf{R}^{k\times N}$ connect the input and $\boldsymbol{\tilde{W}}^{l}\in\mathbf{R}^{N \times k}$ connect out. All the weights $\tilde{w}^{l}_{ij}$ and $v^{l}_{ij}$ random initialized with uniform distribution $\mathrm{U}[-1, 1]$ and i.i.d. $\boldsymbol{B}^{l}\in\{0, 1\}^{k \times N}$ is the mask for matrix $\boldsymbol{V}$, $\sum_{i, j}\left\|B^{l}_{ij}\right\|_{0} \leq N^2, \sum_{j}\left\|B^{l}_{ij}\right\|_{0} \leq N$. Then, let the function of equivalent network be $\tilde{G}(\boldsymbol{s})=\tilde{G}^{L}\circ\tilde{G}^{L-1}\circ\cdots\circ\tilde{G}^{1}(\boldsymbol{s}) $ and $\tilde{G}^{l} = \tilde{\sigma}((\boldsymbol{\tilde{W}}^{l}\odot\boldsymbol{B}^{l})\tilde{\sigma}(\boldsymbol{V}^{l}\boldsymbol{s}))$, where input spiking $\boldsymbol{s}$ is a vector that $\boldsymbol{s}\in\{0, 1\}^{N}$. And the target network is $\hat{G}(\boldsymbol{s})=\hat{G}^{L}\circ\hat{G}^{L-1}\circ\cdots\circ\hat{G^{1}}(\boldsymbol{s})$, where $\hat{G}^{l}(\boldsymbol{s})=\hat{\sigma}(\boldsymbol{\hat{W}}^{l}\boldsymbol{s})$. $l=1, 2, \cdots, L$. $C$ is the constant depending on the supremum probability density of the dataset of the network. Then, $\forall \skx{0<}delta \skx{\leq 1}, \exists \epsilon \skx{>0}$ when $$k \geq N^2\lceil\frac{N}{C_{th}\epsilon}\log{\frac{N^2 L}{\delta-NCL\epsilon}}\rceil,$$ there exists a mask matrix $\boldsymbol{B}$ that $$\left\|\tilde{G}(\boldsymbol{s})-\hat{G}(\boldsymbol{s})\right\| = 0$$ w.p at least $1-\delta$.
\end{lemma}

\begin{proof}
We inherit the event expression $\mathbb{E}_{j, i, k', \epsilon}$ and $\mathbb{E}_{fire}$ from the proof of Lemma\Cref{lemma:LAA_appendix} with a delight modify. Here we add subscript $l$ to denote the layer of the target network. Thus we have:
\begin{equation}
\begin{aligned}
&P((\bigcup_{l, j, i}\mathbb{E}_{l, j, i, k', \frac{\epsilon}{N}})\bigcup\mathbb{E}_{fire, l})\\
\leq &\sum\limits_{l, j, i}P(\mathbb{E}_{l, j, i, k', \frac{\epsilon}{N}})+\sum_{l}P(\mathbb{E}_{fire, l}) \leq LN^{2}\exp({-k' C_{th} \frac{\epsilon}{N}}) + LNC\epsilon\\
\leq &\delta,
\end{aligned}
\end{equation}
Thus, 
\begin{equation}
k \geq N^2\lceil\frac{N}{C_{th}\epsilon}\log{\frac{N^2 L}{\delta-NCL\epsilon}}\rceil,
\end{equation}
and:
\begin{equation}
P(((\bigcup_{l, j, i}\mathbb{E}_{l, j, i, k', \frac{\epsilon}{N}})\bigcup\mathbb{E}_{fire, l})^{\mathsf{C}}) \geq 1 - \delta
\end{equation}
the lemma holds.
\end{proof}

\begin{theorem}
\label{theorem:LTH_appendix}
\textbf{All Steps Approximation.} Fix the weight matrix $\boldsymbol{\hat{W}}^{l} \in [-\frac{1}{\sqrt{N}}, \frac{1}{\sqrt{N}}]^{N \times N}$ which is the connection in target network between a layer of spiking inputs and the next layer of neurons. The equivalent structure is $k$ spiking neurons with $k \times N$ weights $\boldsymbol{V}^{l}\in\mathbf{R}^{k\times N}$ connect the input and $\boldsymbol{\tilde{W}}^{l}\in\mathbf{R}^{N \times k}$ connect out. All the weights $\boldsymbol{\tilde{w}}^{l}_{ij}$ and $v^{l}_{ij}$ random initialized with uniform distribution $\mathrm{U}[-1, 1]$ and i.i.d. $\boldsymbol{B}^{l}\in\{0, 1\}^{k \times N}$ is the mask for matrix $\boldsymbol{V}^{l}$, $\sum_{i, j}\left\|B^{l}_{ij}\right\|_{0} \leq N^2, \sum_{j}\left\|B^{l}_{ij}\right\|_{0} \leq N$. Then, let the function of equivalent network at timestep $t$ be $\tilde{G}^{t}(\boldsymbol{S})=\tilde{G}^{t, L}\circ\tilde{G}^{t, L-1}\circ\cdots\circ\tilde{G}^{t, 1}(\boldsymbol{S}) $ and $\tilde{G}^{t, l} = \tilde{\sigma}^{t}((\boldsymbol{\tilde{W}}^{l}\odot\boldsymbol{B}^{l})\tilde{\sigma}^{t}(\boldsymbol{V}^{l}\boldsymbol{S}^{t}))$, where input spiking $\boldsymbol{S}$ is a tensor that $\boldsymbol{S}\in\{0, 1\}^{N\times T}$. And the target network at timestep $t$ is $\hat{G}^{t}(\boldsymbol{S})=\hat{G}^{t, L}\circ\hat{G}^{t, L-1}\circ\cdots\circ\hat{G^{t, 1}}(\boldsymbol{S})$, where $\hat{G}^{t, l}(\boldsymbol{S})=\hat{\sigma}^{t}(\hat{W}^{l}\boldsymbol{S}^{t})$. $l=1, 2, \cdots, L, t=1, 2, \cdots, T$. $C$ is the constant depending on the supremum probability density of the dataset of the network. Then, $\forall \skx{0<}\delta \skx{\leq 1}, \exists \epsilon \skx{>0} $ when 
$$
k \geq N^2\lceil\frac{N}{C_{th}\epsilon}\log{\frac{N^2 L}{\delta-NCLT\epsilon}}\rceil, 
$$
there exists a mask matrix $\boldsymbol{B}$ that $$\left\|\tilde{G}(\boldsymbol{S})-\hat{G}(\boldsymbol{S})\right\| = 0$$
w.p at least $1-\delta$.
\end{theorem}

\begin{proof}
We inherit the event expression $\mathbb{E}_{l, j, i, k', \epsilon}$ and $\mathbb{E}_{fire, l}$ from the proof of Lemma\Cref{lemma:ALA_appendix} with a delight modify. Here we add subscript $t$ to $\mathbb{E}_{fire, l}$ to denote the layer of the target network. Thus we have:
\begin{equation}
\begin{aligned}
&P((\bigcup_{t, l, j, i}\mathbb{E}_{t, l, j, i, k', \frac{\epsilon}{N}})\bigcup\mathbb{E}_{fire, t, l})\\
\leq &\sum\limits_{t, l, j, i}P(\mathbb{E}_{t, l, j, i, k', \frac{\epsilon}{N}})+\sum_{t, l}P(\mathbb{E}_{fire, t, l}) \\
\leq &LN^{2}\exp({-k' C_{th} \frac{\epsilon}{N}}) + TLNC\epsilon \\
\leq &\delta,
\end{aligned}
\end{equation}
Thus, 
\begin{equation}
k \geq N^2\lceil\frac{N}{C_{th}\epsilon}\log{\frac{N^2 L}{\delta-NCTL\epsilon}}\rceil,
\end{equation}
and:
\begin{equation}
P(((\bigcup_{l, j, i}\mathbb{E}_{t, l, j, i, k', \frac{\epsilon}{N}})\bigcup\mathbb{E}_{fire, t, l})^{\mathsf{C}}) \geq 1 - \delta
\end{equation}
the lemma holds.
\end{proof}

\section{Derivation of Equation~(\ref{eq:error_experctation})}\label{app_sec:pruning_method}

We design a new criterion to prune weights according to their probabilities of affecting the output of spiking neurons. Based on the probabilistic modeling (\Cref{theorem:theorem_in_chapter_3}), for each weight, its influence on the output of the spiking neuron has a probability upperbound $P$, which can be estimated as:
\begin{equation}
P\approx E{(\left|u^{\prime} - u \right|)}\mathcal{N}(0|\mu-u_{th}, {var}),
\end{equation}
where $E{(\left|u^{\prime} - u \right|)}$ is the expectation of the error in the membrane potential brought about by pruning (i.e., effect of weights on linear transformations). In our method, it is written as:
\begin{equation}
E{(\left|u^{\prime} - u \right|)} = \frac{E_{act}[w]\gamma}{\sqrt{\sigma_{\mathcal{B}}^{2}+\epsilon}},
\label{eq:error_experctation_app}
\end{equation}
where $E_{act}[w]$ is the expectation that a weight is activated (the pre-synaptic spiking neuron outputs a spike), $\gamma$ and $\sigma_{B}$ are a pair of hyper-parameters in Batch Normalization \cite{ioffe2015batch}. The detailed derivation of \Cref{eq:error_experctation_app} can be found in \Cref{app_sec:pruning_method}. Next, following \cite{zheng2021going,guo2022recdis}, here we suppose the membrane potential follows Gaussian distribution $\mathcal{N}(\mu, {var})$. Consequently, $\mathcal{N}(0|\mu-u_{th}, {var})$ represents the probability density when the membrane potential is at the threshold $u_{th}$ (i.e., effect of weights on binary nonlinear transformations). Finally, for each weight, we have a probability $P$, and we prune according to the size of $P$ (from small to large).

Now we explain that how we get the \Cref{eq:error_experctation_app}, i.e., \Cref{eq:error_experctation} in the main text. The expression of Batch Normalization (BN) can be written as:

\begin{align}
\widehat{x}_{i} & \leftarrow \frac{x_{i}-\mu_\mathcal{B}}{\sqrt{\sigma_{\mathcal{B}}^{2}+\epsilon}}, \\
y_{i} & \leftarrow \gamma \widehat{x}_{i}+\beta \equiv \mathrm{BN}_{\gamma, \beta}(x_{i}).
\end{align}

When a spike $s$ passes through a weight in convolution layer and the BN layer, the output is:
\begin{equation}
 \frac{sw\gamma}{\sqrt{\sigma_{\mathcal{B}}^{2}+\epsilon}}+(\frac{-\mu_\mathcal{B}\gamma}{\sqrt{\sigma_{\mathcal{B}}^{2}+\epsilon}} + \beta).
\end{equation}

We know that these two operations can be regard as linear transformation, where the input spike $s$ multiplies the scaling weight $\frac{w\gamma}{\sqrt{\sigma_{\mathcal{B}}^{2}+\epsilon}}$ and add a bias $(\frac{-\mu_\mathcal{B}\gamma}{\sqrt{\sigma_{\mathcal{B}}^{2}+\epsilon}} + \beta)$. Since the bias will not change after changing the weight, the error of membrane potential only related to the scaling weight $\frac{w\gamma}{\sqrt{\sigma_{\mathcal{B}}^{2}+\epsilon}}$. Thus, we design $E{(\left|u^{\prime} - u \right|)}$ as follows:
\begin{equation}
E{(\left|u^{\prime} - u \right|)} = E{(\frac{\left|sw\gamma\right|}{\sqrt{\sigma_{\mathcal{B}}^{2}+\epsilon}})} = \frac{\left|\gamma\right|}{\sqrt{\sigma_{\mathcal{B}}^{2}+\epsilon}}E{(\left|sw\right|)} = \frac{E_{act}[\left|w\right|]\left|\gamma\right|}{\sqrt{\sigma_{\mathcal{B}}^{2}+\epsilon}}.
\end{equation}

Note, in keeping with the notations in classic BN \cite{ioffe2015batch}, there is some notation mixed here, but only for the derivation of \Cref{eq:error_experctation}.

\begin{algorithm*}[!th]
\label{alg:imp}
   \caption{Rewind Iterative Magnitude Pruning (IMP).}
   \label{alg:example}
\begin{algorithmic}
   \STATE {\bfseries Input:} pruning rate $p$, iterations $K$, the rewind epoch $R$, max training epoch $N$
   \STATE Train network for $R$ epochs.
   \STATE Save the parameters as rewind parameters $\theta_{rewind}$
   \STATE Train network for $n$ epochs.
   \FOR{$k=1$ {\bfseries to} $K$}
   \STATE Prune network by remove $p\%$ of the lowest magnitude nonzero weights of the network base on a metrics, get the mask $m_{k}$
   \STATE Reload the rewind parameters and mask them. ($\theta_{rewind}\odot m_{k}$)
   \STATE Train the network for $N$ epochs.
   \ENDFOR
\end{algorithmic}
\end{algorithm*}

\section{Implementation Details of Sub-Network Search}
\label{app_sec:exp_details}
All training methods strictly follow experiments in \cite{LTH_SNN_2022}, and the sub-network search module follows \cite{ramanujan2020s}. For the two datasets of CIFAR10/100, we use cosine learning scheduling and SGD optimizer with momentum 0.9 and weight decay 5e-4, and the total number of training epochs is 300. The learning rate is set to 0.3. batch size is set to 128. The timestep $T$ of SNN is 5. We simply replace all the weight modules in network with the corresponding sub-network search modules in \href{https://github.com/allenai/hidden-networks/blob/master/simple_mnist_example.py}{hidden-networks/blob/master/simple\_mnist\_example.py}

\section{Iterative Magnitude Pruning (IMP) and the proposed SF-IMP Algorithms}\label{app_sec:algorithm}
Among the LTH-based pruning algorithms, the Iterative Magnitude Pruning (IMP) method has good performance. In IMP, the parameter $\theta\in\mathbf{R}^n$ of network $f(x;\theta)$ is pruned iteratively. (For the sake of briefness and convention of the symbols, unless otherwise specified, the meanings of the symbols in this chapter have nothing to do with the previous chapters.) We set $K$ iterations. In the $k$-th iteration, we first train the network till convergence, then prune $p\%$ of nonzero parameters of $\theta_{trained}\odot m_{k-1}$ by mask $m_{k}\in\{0, 1\}^{n}$. Then reinitialize the network with parameter $\theta_{init}\odot m_{k}$ and repeat the operations until the $K$-th iteration ends. 

For SF-IMP-1, we change the criterion from considering magnitude to:
\begin{equation}
\frac{E_{act}[\left|w_{ij}\right|]\left|\gamma_{j}\right|}{\sqrt{\sigma_{\mathcal{B}j}^{2}+\epsilon}}\mathcal{N}(0|\mu_{j}-u_{th}, {var}_{j}),
\end{equation}
Here, $i$ is the index of the input channel while $j$ is the index of the output channel. We statistic the firing frequency for each input channel to evaluate $E_{act}[\left|w_{ij}\right|]$, and other parameters $\gamma_{j}, \sigma_{\mathcal{B}j}, \mu_{j}, {var}_{j}$ are statistic by every output channel. However, the encoding layer and the Fully Connected (FC) layer are not suited for this algorithm, thus we keep the magnitude criterion for these two layers.

For SF-IMP-2, since the sparsity of the encoding layer and the FC layer have a great influence on pruning, we apply a dynamic strategy for these layers \cite{guo2016dynamic}. Specifically, we only mask the weight when computing loss, while when updating weight and re-initialize weight, the mask is not used. 

\end{document}